\documentclass[11pt,letterpaper]{article}
\usepackage{url}
\usepackage{fullpage}
\usepackage[subtle,mathspacing=normal]{savetrees}
\usepackage{amsmath}
\usepackage{amssymb}
\usepackage{amsthm}
\usepackage{mathtools}
\usepackage{accents}
\usepackage{thm-restate}

\usepackage{bm}
\usepackage{color}
\usepackage{array}
\usepackage{xy}
\usepackage{setspace}
\usepackage{varwidth}      
\usepackage{multicol}
\usepackage{mdframed}
\usepackage{enumitem}
\usepackage{ragged2e}
\usepackage{microtype}
\usepackage{xspace}

\usepackage{graphicx}
\usepackage{subcaption}
\usepackage{booktabs}      
\usepackage{threeparttable}

\usepackage{tikz}
\usepackage{standalone}
\usepackage{todonotes}

\usepackage{hyperref}
\usepackage[capitalize]{cleveref}

\usepackage[
    backend=bibtex,
    style=alphabetic,
    natbib=true,
    maxalphanames=6,
    minalphanames=4,
    maxnames=10,
]{biblatex}
\usepackage{complexity}

\usetikzlibrary{positioning, shapes.geometric, arrows.meta, arrows, calc,
                fit, backgrounds, matrix, chains, decorations,
                decorations.pathreplacing}

\tikzstyle{system}=[rectangle, draw, fill=lightgray, minimum height=0.8cm,
                    minimum width=0.8cm, thick]
\tikzstyle{BC}=[system]
\tikzstyle{resource}=[system]
\tikzstyle{RO}=[resource, minimum width=1cm]
\tikzstyle{protocol}=[circle, inner sep=0.7mm, draw]
\tikzstyle{simulator}=[circle, inner sep=0.7mm, draw]
\tikzstyle{memory}=[resource]
\tikzstyle{distinguisher}=[resource, fill=white, minimum width=3.5cm,
                           minimum height=1.2cm]
\tikzstyle{link}=[]

\newtheorem{thm}{Theorem}[section]

\theoremstyle{remark}

\newtheorem{theorem}[thm]{Theorem}
\newtheorem{definition}[thm]{Definition}
\newtheorem{lemma}[thm]{Lemma}
\newtheorem{proposition}[thm]{Proposition}
\newtheorem{claim}[thm]{Claim}

\theoremstyle{remark}
\newtheorem{remark}[thm]{Remark}

\newcommand{\defn}[1]{\textbf{\emph{#1}}}             
\renewcommand{\paragraph}[1]{\vspace{.2 cm} \noindent \textbf{#1}}

\newmdenv[
    skipabove=\smallskipamount,
    skipbelow=\smallskipamount,
    innertopmargin=6pt,
    innerbottommargin=6pt,
    innerleftmargin=8pt,
    innerrightmargin=8pt,
    linewidth=0.8pt,
    roundcorner=2pt,
]{protocolbox}

\newcommand{\negl}{\mathsf{negl}}        
\newcommand{\algo}[1]{\mathsf{#1}}       
\newcommand{\tO}{\widetilde{O}}          

\newcommand{\NN}{\mathbb{N}}
\newcommand{\bbN}{\mathbb{N}}            
\newcommand{\FF}{\mathbb{F}}            
\newcommand{\bin}{\{0,1\}}
\newcommand{\Set}[1]{\left\{ #1 \right\}}
\newcommand{\eps}{\varepsilon}          

\newcommand{\bx}{\mathbf{x}}
\newcommand{\bu}{\mathbf{u}}

\newcommand{\probcond}[2]{
    \Pr\left[
    \begin{array}{c}
    #1
    \end{array}
    :
    \begin{array}{c}
    #2
    \end{array}
    \right]
}

\newcommand{\secp}{\lambda}
\renewcommand{\sec}{\secp}              

\newcommand{\secparam}{1^\secp}

\newcommand{\calF}{\mathcal{F}}        
\newcommand{\calM}{\mathcal{M}}
\newcommand{\Rel}{\mathcal{R}}         

\newcommand{\Prover}{\mathcal{P}}
\newcommand{\prove}{\Prover}
\newcommand{\MProver}{\widetilde{\Prover}} 
\newcommand{\Verifier}{\mathcal{V}}
\newcommand{\ver}{\Verifier}

\newcommand{\calP}{\mathcal{P}}
\newcommand{\calV}{\mathcal{V}}
\newcommand{\calA}{\mathcal{A}}
\newcommand{\calE}{\mathcal{E}}

\newcommand{\Setup}{\mathsf{Setup}}
\newcommand{\Extract}{\mathsf{Ext}}
\newcommand{\ext}{\Extract}

\newcommand{\Open}{\mathsf{Open}}

\newcommand{\Commit}{\mathsf{Commit}}

\newcommand{\pp}{\mathsf{pp}}          
\newcommand{\hk}{\mathsf{hk}}          
\newcommand{\rt}{\mathsf{rt}}          
\newcommand{\tr}{\mathsf{tr}}          
\newcommand{\st}{\mathsf{st}}          
\newcommand{\cval}{\mathsf{c}}         
\newcommand{\PC}{\mathsf{PC}}          

\newcommand{\smht}{\mathsf{ht}}
\newcommand{\htgen}{\mathcal{G}_\smht}
\newcommand{\hthash}{\mathcal{H}_\smht}

\newcommand{\htprove}{\mathcal{P}_\smht}
\newcommand{\htver}{\mathcal{V}_\smht}

\newcommand{\kil}{\mathsf{Kil}}
\newcommand{\kilsetup}{\Setup_\kil}
\newcommand{\kilprove}{\prove_\kil}
\newcommand{\kilver}{\ver_\kil}
\newcommand{\kilext}{\ext_\kil}

\newfunc{\seBARG}{seBARG}

\newfunc{\RAM}{RAM}
\newfunc{\RAMDel}{RAMDel}
\newfunc{\RAMSNARG}{RAM~SNARG}
\newfunc{\SNARG}{SNARG}
\newfunc{\SNAP}{SNAP}
\newfunc{\flSEH}{flSEH}
\newfunc{\HT}{HT}
\newfunc{\Digest}{Digest}
\newfunc{\digest}{digest}
\newfunc{\node}{node}
\newfunc{\layer}{layer}

\newfunc{\RightSize}{RightSize}
\newfunc{\Neighbor}{Neighbor}
\newfunc{\Sketch}{Sketch}
\newfunc{\Recover}{Recover}

\DeclarePairedDelimiter{\lrag}{\langle}{\rangle}
\DeclarePairedDelimiter{\abs}{\lvert}{\rvert}
\renewcommand{\lang}{\mathcal{L}}

\newcommand{\sigmaIPP}{\sigma + \log(2d)}

\newcommand{\SIS}{\mathsf{SIS}}

\newcommand{\cksum}{\ensuremath{\mathsf{cksum}}\xspace}
\newcommand{\cksumT}{{\ensuremath{R_\cksum}\xspace}}

\title{How to Avoid Debate: \\ Scalable AI Safety via Doubly-Efficient Interactive Proofs}
\author{
Liyan Chen\thanks{MIT. \texttt{cliyan@mit.edu}.} 
    \and 
Yael Tauman Kalai\thanks{MIT. \texttt{tauman@mit.edu}.} 
    \and 
Zoe Xi\thanks{MIT. \texttt{zoexi@mit.edu}.}
}
\date{}

\begin{document}

\maketitle

\begin{abstract}

As AI models continue to develop powerful
  capabilities, it becomes critical that we are able to
  verify that their output is aligned with our intentions. A recent line of work focuses on
  verification via debate, a model
  of interactive proofs where two competing powerful provers, or AI
  models, debate each other to convince a weak verifier, or a human,
  of the correctness of their claim. However, debate assumes that the two AI models possess equal abilities and that one of them is truthful, which may not be realistic. 
  
  In this work, we show \emph{how to avoid debate}: we
  initiate the study of \emph{single-prover}
  interactive proofs for AI safety. 
Prior results in single-prover interactive proofs do not immediately
  carry over to the AI safety setting: for example, they do not work when the computation has access to an oracle, such as to human judgment or an external database such as the web.   
  We present doubly-efficient
  single-prover interactive proofs and arguments for oracle-aided computations (also known as relativizing proofs),  in the settings where (1) the
  computation is robust, in the sense that the output does not change
  if at most a small fraction of the answers to oracle queries are
  incorrect, or (2) the oracle is a low-degree polynomial. These results suggest that interactive verification is possible even without debate, under structured or noise-tolerant oracle access.

\end{abstract}  


\newpage
\tableofcontents
\newpage

\section{Introduction}\label{sec:intro}


As machine learning models become increasingly capable, they are
 being trained to perform complex tasks that would be
prohibitively expensive for a human to verify, leading to critical
safety concerns. How can we efficiently supervise the training of an
AI model to perform some complex task in a way that aligns with our
intentions, when we may only have an obscure understanding of the task
ourselves? 

As a concrete example, consider the scenario of training a large
language model (LLM) to write a long high-stakes document, such as a
legal contract.\footnote{This example is due to Brown-Cohen, Irving, and Piliouras
  \cite{brown2024scalable}.} To generate a training label on a contract
produced by the LLM, it is necessary to verify that every passage of
the contract is correct, where correctness here is dictated by human
judgment. However, it would be unreasonably expensive to require a
human to carefully read through a long legal contract to
produce just one training label. In this setting, it is important to
have a training protocol that is extremely efficient in its use of
human judgments. For another example, we might consider a scenario where we train an LLM on a large public dataset (e.g., the web). The LLM generates an output, and we would like to verify that the output is correct while making only very few queries to the dataset. 


Various approaches to this problem, known as \emph{scalable
  oversight}, have been proposed in the AI safety literature
\cite{christiano2018supervising, leike2018scalable, irving2018ai}. In particular, a
recent line of work focuses on designing verifiable 
training protocols by leveraging tools from theoretical cryptography, namely
the area of interactive proofs \cite{irving2018ai, brown2024scalable,
  brown2025avoiding}. An interactive proof \cite{GoldwasserMR85,Babai85} is a
protocol wherein a powerful but untrusted prover interacts with a weak
verifier to try to convince the verifier of the validity of some
statement. For security, if the statement is indeed valid, then the verifier should
accept, and if the statement is not valid, then the verifier should
reject (with high probability). 
At a high level, taking the prover to
be an AI model and the verifier to be a human, this closely resembles
what we want in our safe training setting.

The main reason why prior work in interactive proofs does not
immediately carry over to our setting, however, is that these classical results assume that the computation is given in a well-defined mathematical manner, such as a Turing machine.  In the setting of AI safety, tasks can be modeled as computation with access to an oracle, which is a ``black-box'' function that we can't see the inner workings of; an oracle might represent human judgment or an external database such as the web, for example. Going back to our examples, we would like to have a protocol where the verifier makes few queries to a human expert or to the dataset, which we could model as
an oracle. Interactive proofs for oracle-aided computations are said to {\em relativize}, and we do not have interactive proofs that relativize. At the same time, relativization is believed to be essential to AI safety: the UK AISI Alignment team claims that ``in order to be relevant to AI safety, an interactive proof result must relativize'' \cite{irving2025relativise}. 



This aim of relativization has driven researchers to study a new model of interactive proofs, known as \emph{debate}.
In this model,
introduced by Irving, Christiano, and Amodei in their 2018 paper ``AI
Safety via Debate'' \cite{irving2018ai}, there are two debating
provers, who hold opposing claims, and a verifier, and the provers
debate each other to try to convince the verifier of the validity of
their claim. In the same paper, Irving, Christiano, and Amodei show
that $\mathsf{DEBATE} = \mathsf{PSPACE}$ 
(where $\mathsf{DEBATE}$ is the
class of problems that have a debate interactive proof), and this
result relativizes, whereas the celebrated $\mathsf{IP} = \mathsf{PSPACE}$ result \cite{shamirIP=PSPACE} (where
$\mathsf{IP}$ is the class of problems that have a single-prover
interactive proof), does not; the moral is that to recover
relativization one can introduce a second debating prover. A
follow-up work by Brown-Cohen, Irving, and Piliouras \cite{brown2024scalable} introduces an
analogue of this model of debate, called \emph{doubly-efficient}
debate, where the provers are constrained to run in polynomial time
instead of being computationally unbounded---the motivation is that in
our AI safety setting the provers are AI models, and AI models, while
powerful, are still computationally bounded.

However, debate relies on assumptions that may not be realistic: in order to use debate for AI safety, we need that the two AI models are properly incentivized to debate, that they have the same computational abilities, and that one of them is truthful. This last assumption of truthfulness is particularly concerning: Irving, Christiano, and Amodei note that it hinges on the assumption that it is harder to lie than to refute a lie, that is, that one of the AI models is incentivized to be truthful, which may or may not be true in any
particular setting \cite{irving2018ai}.
Further limitations of debate as an approach for
AI alignment are discussed in \cite{irving2018ai}.

In light of these limitations, in this work we investigate the following question:

\begin{center}
\emph{Can we construct a single-prover interactive proof that relativizes?}
\end{center}





In the ``doubly-efficient'' model that we study, there is a
polynomial-time prover and an even more efficient verifier, who both
have access to an oracle, and the prover tries to convince the
verifier of the correctness of some computation that may depend on
queries to the oracle. Notably, the super-efficient verifier not only
runs in much less time than the prover, but he also makes few queries
to the oracle, much fewer than the number that he would need in order
to perform the computation by himself---in our examples,
this would mean that very few queries to the human expert or to the database are needed.

Unfortunately, Barbara, Chiesa, and Guan \cite{DBLP:conf/tcc/BarbaraCG25} prove that there do not exist interactive proofs that relativize for all computations. Intuitively, their observation is that if the computation uses many oracle gates to a  random oracle, and the output of the computation is unstable, in the sense that it hinges on the answers of all the oracle queries being correct, then the verifier would need to check the correctness of every oracle call, and thus would not be efficient.

\subsection{Our Results}

In this paper, we construct relativizing doubly-efficient interactive proofs and arguments for two natural settings, where the computation is ``robust'' or where the oracle is a low-degree polynomial. Our results suggest we can recover relativization even without two debating provers. 

\paragraph{Robust computation.} We say that a computation is robust if the output does not change even if a small fraction of answers to oracle queries are modified. Intuitively, robustness allows a verifier to ``spot-check'' oracle calls instead of having to evaluate all of them. It seems natural to assume that the computation is robust, especially in the setting where the oracle is taken to represent human judgment: human judgment is already error-prone, so as a safeguard, the output of the computation should not change if a small number of the human answers change, otherwise the output is not very meaningful.
More broadly, many natural tasks may be able to be made robust using redundancy, e.g., in the case of human judgment, we might rephrase the question or ask multiple people and take the majority answer.

\paragraph{Low-degree oracles.}
We also consider computations with access to an oracle that can be represented by a low-degree polynomial. Unlike the above robust setting, here we do not assume anything about the computation, but rely on the strong algebraic structure of the oracle in order to make the verifier efficient. Our protocol in this setting can be viewed as a first step towards designing protocols for computations with access to an oracle that is ``learnable'' by a simple machine learning model class. More concretely, we might take the oracle to represent a database, such as the web: any database $\bin^{\log N} \to \bin$ can be converted into a low-degree oracle $\mathbb{F}^{\log N} \to \mathbb{F}$ via a standard low-degree extension (see Chapter 3 of \cite{DBLP:journals/ftsec/Thaler22}). Then our protocol would allow us to efficiently verify the correctness of any computation that can interact with this database. 

Our protocols exhibit different tradeoffs between soundness and efficiency guarantees. Some protocols achieve statistical soundness, which guarantees security against even computationally unbounded cheating provers, and some achieve only computational soundness, which guarantees security against polynomial-time cheating provers. The protocols have varying efficiency guarantees regarding the number of bits communicated between the prover and the verifier, the running time of the verifier, and the number of queries that the verifier makes to the oracle. For a more detailed explanation of these notions, see \cref{sec:prelims}. Also see Table~\ref{tab:overview} for an overview of our results. 



\begin{table*}[t]
    \centering
    \caption{An overview of our results. We model the computation as an oracle Boolean circuit of depth $D$ with $\ell$ oracle gates divided into $d$ levels. Here $\mathcal{V}_{\text{time}}$ is the verifier's running time, $\mathcal{V}_{q}$ is the number of oracle queries made by the verifier, $r$ is the round complexity, $\sigma$ is a soundness parameter, and $\eps$-robust means that the output of the circuit does not change even if at most an $\eps$-fraction of the answers to oracle queries are changed on any input (see \cref{sec:prelims} for more formal definitions).}
    \begin{threeparttable}
    \setlength{\tabcolsep}{3.55pt}
    \begin{tabular}{lcccccc}
            \toprule
            Protocol & Soundness & $\mathcal{V}_{\text{time}}$ & $\mathcal{V}_q$ & $r$ & Setting & Adaptive? \\
            \midrule
            \cref{thm:proof-nonadaptive} & statistical & $\widetilde{O}(\sigma(\eps\ell D + 1/\eps + n))$ & $O(\sigma/\eps)$ & $\widetilde{O}(D)$ & $\eps$-robust & no \\
            \cref{thm:proof-adaptive} & statistical & $\widetilde{O}(\sigma d(\eps\ell D + 1/\eps + n))$ & $\widetilde{O}(\sigma d/\eps)$ & $\widetilde{O}(dD)$ & $\eps$-robust & yes \\
            \cref{thm:arg-adaptive} & computational & $\widetilde{O}(1/\eps + n)$ & $O(1/\eps)$ & $O(1)$ & $\eps$-robust & yes \\
            \cref{thm:low-degree} & computational & $\widetilde{O}(n)$ & 1 & $\widetilde{O}(1)$ & $\poly$-degree oracle & yes \\
            \bottomrule
        \end{tabular}%

    
    \end{threeparttable}

    \label{tab:overview}
\end{table*}



\paragraph{Outline of the paper.} We begin in \cref{sec:tech} with an overview of the key ideas behind our results. In \cref{sec:prelims}, we introduce our model of doubly-efficient interactive proofs for oracle-aided computation and more definitions that we need. In \cref{sec:robust} and \cref{sec:low-deg}, we present our results for the settings where the computation is robust and the oracle is a low-degree polynomial, respectively.  We end in \cref{sec:conclusion} with some open questions.  

\subsection{Related Work}


\paragraph{AI safety via debate.} The work most closely related to
ours is Brown-Cohen, Irving, and Piliouras' ``Scalable AI Safety Via
Doubly Efficient Debate,'' which presents doubly-efficient
\emph{debate} protocols for oracle-aided computation
\cite{brown2024scalable}; we study doubly-efficient
\emph{single-prover} protocols for oracle-aided computation. Debate
was first proposed as an approach to scalable oversight in 2018 by
Irving, Christiano, and Amodei in their paper ``AI Safety via Debate''
\cite{irving2018ai}, which assumes that the two debating provers are
computationally unbounded. More recently, work on debate has focused
on limitations, e.g., the ``obfuscated arguments problem,'' wherein a dishonest debater
can come up with a flawed argument where the flaw is very hard to find
\cite{barnes2020obfuscated, brown2025avoiding, buhl2025alignment}.

\paragraph{Interactive proofs for trustworthy machine learning.} Recently, there has been a line of work on applying interactive proofs to problems of trust in machine learning \cite{amit2024models, hammond2024neural, anil2021learning, wäldchen2024interpretabilityguaranteesmerlinarthurclassifiers, kirchner2024prover}. The most closely related work in this line is that of Amit, Goldwasser, Paradise, and Rothblum \cite{amit2024models} 
who introduce self-proving models, which are models that, in addition to providing answers, prove the correctness of those answers to a verifier using an interactive proof (with standard soundness and relaxed distributional completeness guarantees). 
Their focus is on training models to produce correctness proofs for computations that already admit interactive proofs. 
In contrast, our work expands the class of doubly efficient interactive proofs and arguments to oracle-aided computations.

\paragraph{Delegating computation.}
The problem of supervising complex computation with a weak verifier has been studied extensively in computational complexity and cryptography. Besides the celebrated $\mathsf{IP}=\mathsf{PSPACE}$ theorem \cite{shamirIP=PSPACE} and the $\mathsf{PCP}$ theorem \cite{arora1998probabilistic}, there has been a line of work studying doubly-efficient interactive proofs, motivated by the application of delegating computation \cite{goldwasser2015delegating, STOC:ReiRotRot16, berger2025efficientlybatchingunambiguousinteractive}. On the other hand, if we only consider security against polynomial-time algorithms, there has been a long line of work on constructing protocols for super-efficient verification for any $\mathsf{NP}$ statement, starting with the seminal work of Kilian and Micali \cite{kilian1992note, micali2000computationally}. These protocols, known as SNARGs or SNARKs, are already widely used in practice (see, e.g., \cite{SP:BCGGMT14}). Our work can be regarded as an extension of both of the above lines, to the setting of oracle-aided computation.

\section{Technical Overview}\label{sec:tech}


Here, let us sketch the key ideas behind our results.

In all of our protocols, we model the computation that the verifier
wishes to verify as an oracle Boolean circuit, i.e., a Boolean circuit
$C$ that may
contain gates that compute an arbitrary function $O: \{0, 1\}^*\to\{0, 1\}$ thought of as an oracle.  We call such gates oracle gates and denote the circuit by $C^O$. 
Given input $x$, a succinct description of circuit
$C$, and access to oracle $O$, the prover and verifier engage in a
protocol to prove that $C^O(x) = 1$.

Throughout this overview, we use $\widetilde{O}(\cdot)$ to hide polylogarithmic factors as well as, for simplicity, factors depending on the soundness parameter $\sigma$, the security parameter $\secp$, and the oracle query length.

\subsection{Proof system for robust oracle computation}

\paragraph{A starting point: an IPP where the verifier has unreliable
  access to the input.} Our proof system in the robust setting makes
key use of an 
interactive proof of proximity (IPP) \cite{rothblum2013interactive, rothblum2020batch}, which is a
doubly-efficient interactive proof where the verifier is extremely
efficient, running in only sublinear time in the input size, and the soundness guarantee
is relaxed so that the verifier is only required to reject inputs that
are ``far'' from any string in the language in, say, fractional
Hamming distance. Notably, since the verifier runs in sublinear time,
he cannot even read the entire input: in an IPP, the verifier is
assumed to have query access to the input, meaning that he can read
any bit in one timestep. 

In our setting, we would like the verifier to
make very few oracle queries, sublinear in the number of oracle gates
in the circuit, which seems to resemble the IPP setting. One natural
idea would be to have the prover and verifier engage in an IPP to show
that $C_Q(x) = 1$, where $Q$ consists of queries that $C$ makes to $O$ on input $x$ along with the corresponding answers, and $C_Q(x)$ runs $C(x)$ but with the
inputs and outputs of oracle gates replaced by those in $Q$. The
verifier in this protocol would only have to read a few bits of $Q$,
which is exactly what we want. Soundness follows from the relaxed
soundness guarantee under the assumption that $C$ is robust.

But this protocol doesn't work as is, since we haven't specified how
the verifier accesses bits of $Q$. In an IPP, the verifier is assumed
to have reliable access to $Q$ (i.e., $Q$ is written down in advance),
but in our setting the verifier does not have reliable access to $Q$. Instead,
at the end of their interaction,\footnote{We assume that the verifier can make all of his queries to the input at the end of the
  interaction.  This is indeed the case in both known IPP constructions in the literature \cite{rothblum2013interactive, rothblum2020batch}.} we have the verifier ask the unreliable prover to
send the bits in $Q$ that he wishes to query, but since the prover
might lie, we then have the verifier check that the prover's answers
are correct (or ``close'' to being correct). If
the verifier requests position $i$ in $Q$ and the prover sends
$(x_i, y_i)$, then the verifier can easily check that $y_i = O(x_i)$
by making one oracle call, but it might be inefficient for the
verifier to check that $x_i$ is computed correctly.

\paragraph{Recovering reliable access.} In the case where the circuit
$C$ makes ``nonadaptive'' oracle queries, meaning that queries are not
allowed to depend on answers to other oracle queries, we can fix this
problem by having the prover and verifier also engage in a
doubly-efficient interactive proof (DEIP) to show that the $x_i$'s
that the prover sends are computed correctly from $x$. In the general
setting where $C$ may make adaptive oracle queries, i.e., queries may
depend on answers to other oracle queries, this fix no longer works,
since the computation that we would like to do a DEIP for may now
contain oracle gates. Namely, in the nonadaptive setting, each $x_i$ can only
depend on $x$, while in the general setting, each $x_i$ may also
depend on the answers to several oracle queries, which may in turn
depend on the answers to several oracle queries, and so on.

In the general setting, the idea is to instead combine this IPP idea
with recursion. After the initial IPP to prove that $C_Q(x) = 1$, one
natural idea is to do another IPP to prove that the prover's message containing $Q$ is computed correctly, and
so on. As we explain below, this idea does not work as is, but does work if we add a minor tweak.

In what follows, we view the circuit $C$ as containing $d$ adaptivity levels of
oracle gates, where queries can only depend on answers corresponding
to oracle gates in smaller levels. Assume for simplicity that $d = 2$
(the argument for general $d$ follows the same structure), and let
$Q_1$ and $Q_2$ denote the substrings of $Q$ corresponding to levels
$1$ and $2$, respectively. 

\paragraph{First attempt.}  After the initial IPP for proving that
$C_Q(x) = 1$, the verifier queries the prover on some positions in
$Q$, which we can separate into queries to $Q_1$ and queries to $Q_2$,
and the prover sends back answers. Letting $A_2$ denote the prover's
answers to $Q_2$, the prover and verifier engage in an IPP to show
that the queries in $A_2$ are computed correctly from $x$ and
$Q_1$. At the end of this IPP, the verifier queries the prover for
some positions in $Q_1$, and the prover sends back answers. Letting
$A_1$ denote the prover's answers to $Q_1$ across both IPPs, the
prover and verifier engage in a DEIP to show that the queries in $A_1$
are computed correctly from $x$. The verifier also checks himself that
all the oracle calls in $A_1$ and $A_2$ are computed correctly.

This protocol doesn't work as is: the problem is that the computation that computes the answers for $Q_2$ is an oracle-aided computation that is not necessarily robust. Specifically, assuming that the input $Q_1$ to
the second IPP is fixed (which in some sense is the case, since we
check using a DEIP that $A_1$ is computed correctly from $x$),
soundness of the second IPP only guarantees that there is some string
$Q_1'$ close to $Q_1$ such that $A_2$ is computed correctly from
$Q_1'$ and $x$, which means that there are potentially many different
strings that the prover can choose to set $A_2$ to. So the input to
the initial ``IPP'' is not fixed in advance (which means that the
initial ``IPP'' is not in fact an IPP). 

\paragraph{Our fix.}  To fix this, we have the
prover commit (in a weak sense) to $Q_1$ and to $Q_2$ at the beginning of the protocol using a checksum: the prover computes checksums of $Q_1$ and $Q_2$ and sends these to the verifier. Then in each IPP, the prover additionally proves that the
checksums of the inputs are correct: we change the initial IPP to prove that
$C_Q(x) = 1$ and that the checksums of $Q_1$ and $Q_2$ are the same as the prover's checksums, and
the second IPP to prove that $A_2$ is computed correctly from $x$ and
$Q_1$ and that the checksum of $Q_1$ is the same as the prover's checksum. In this $d = 2$
example, and for general $d$, we can show inductively that if the
verifier accepts, then the input to each IPP is ``fixed'' assuming the
inputs to IPPs corresponding to lower levels are ``fixed,'' taking the
IPP on input $Q_1$ as a base case, so that the ``IPPs'' in the
protocol are in fact IPPs (albeit on inputs which may not be the
claimed inputs, but are in some sense close to the claimed
inputs). Under the assumption that the circuit is robust, the
soundness guarantees of the IPPs and DEIP, the security of the checksum, and the verifier's checks together imply
soundness.

\paragraph{Complexity analysis.} The number of IPPs is equal to the adaptivity level, $d$, and each IPP corresponds to a depth $\leq D$ computation. Suppose that the circuit $C$ is $\eps$-robust, meaning that the output of $C$ does not change even if at most an $\eps$ fraction of answers to oracle queries are changed (see \cref{defn:robustness} for a formal definition). Using \cite{rothblum2020batch} (see \Cref{rothblumx2ipp}) and that $C$ is $\eps$-robust, we can implement each IPP to consist of at most $\widetilde{O}(D)$ rounds, with query complexity $\widetilde{O}(1/\eps)$, communication complexity $\widetilde{O}(\eps\ell D)$, and verifier running time $\widetilde{O}(\eps\ell D + 1/\eps + n)$. At the end, the prover and verifier run a DEIP corresponding to a depth $\leq D$ computation, which, using \cite{goldwasser2015delegating} (see \Cref{deIP}), consists of $\widetilde{O}(D)$ rounds, with communication complexity $\widetilde{O}(D)$ and verifier running time $\widetilde{O}(D+n)$. Putting everything together, we get that the protocol consists of $\widetilde{O}(dD)$ rounds, with query complexity $\widetilde{O}(d/\eps)$, communication complexity $\widetilde{O}(\eps\ell dD)$, and verifier running time $\widetilde{O}(d(\eps\ell D + 1/\eps + n))$.



\subsection{Argument system for robust oracle computation}

We also give an argument system for robust oracle circuits that
achieves better parameters than the proof system, by leveraging tools
from cryptography. Recall that an argument
system guarantees that soundness holds  only against polynomial-time provers, while
the soundness guarantee of a proof system must hold even against
computationally unbounded provers (so essentially argument systems
trade security for efficiency).

In our argument system, the prover uses a Merkle hash to commit to a
string of queries and answers, which is supposed to be the string $Q$
of queries and answers that $C$ makes to $O$ on input $x$. The
prover and verifier engage in a succinct argument of knowledge \cite{kilian1992note} to prove
that there exists a string $Q'$ such that (i) $Q'$ is consistent with
the commitment and (ii) $C_{Q'}(x) = 1$, where $C_{Q'}(x)$ runs $C(x)$
but with the inputs and outputs of oracle gates replaced by those in
$Q'$. The additional knowledge property ensures that there not only
exists a $Q'$ satisfying (i) and (ii), but that the prover also knows
such a $Q'$, which we can efficiently ``extract'' out. The verifier
also asks the prover to open the commitment at a few random positions (specifically, at $O(1/\eps)$ positions, where we assume the circuit $C$ is ``$\eps$-robust'')
and checks that the oracle calls at these positions are correct. If
the verifier accepts, then we can show that with high probability
there is some fixed string $Q'$ such that the computation ``with
respect to'' $Q'$ is correct, and most of the oracle calls in $Q'$ are
correct; then using the robustness assumption, we get soundness.

Using the ``succinctness property'' of the succinct argument of knowledge \cite{kilian1992note} (see \cref{thm:kilian}), we get that the communication complexity of this protocol is dominated by the communication in the step where the prover opens the commitment at $O(1/\eps)$ positions, which is $\widetilde{O}(1/\eps)$. Similarly, the verifier's running time is $\widetilde{O}(1/\eps + n)$. The query complexity is $O(1/\eps)$, and the round complexity is $O(1)$.  

\subsection{Argument system for computation with access to a low-degree oracle}

We also give an argument system for a different setting, where the
circuit (which we no longer assume is robust) has access to an oracle
that can be represented by a low-degree polynomial. In this protocol,
we have the prover use a polynomial commitment scheme to commit to two
polynomials $G$ and $F$, which are supposed to encode the strings of
queries and answers that $C$ makes to $O$ on input $x$,
respectively. The prover uses a succinct argument of knowledge \cite{kilian1992note} to
convince the verifier that she knows $G$ and $F$ such that if we run
$C(x)$ but with the queries and answers to $O$ replaced with those
encoded by $G$ and $F$, then the output is $1$. The verifier also asks
the prover to open the commitments at a random field point $u$ and
checks that $F(u) = O(G(u))$. Letting $G^*$ and $F^*$ denote the
polynomials that we can extract out of the succinct argument of
knowledge (using the knowledge property), if $F^*$ does not correctly
encode the oracle answers corresponding to $G^*$, then using that
$F^*$ and $G^*$ are low-degree and applying Schwartz-Zippel, we get
that w.h.p.  $F^*(u)\neq O(G^*(u))$, so we achieve
soundness. Notably, the verifier in this protocol only makes one query
to the oracle.

Applying the guarantees of the succinct argument of knowledge \cite{kilian1992note} (see \cref{thm:kilian}) and the polynomial commitment scheme \cite{C:CMNW24} (see \cref{thm:prelim_poly_commitment}) that we use, we get that this protocol has round complexity $\widetilde{O}(1)$ and communication complexity $\widetilde{O}(1)$. The verifier runs in time $\widetilde{O}(n)$, which consists of the verification in Kilian's protocol and the polynomial commitment scheme, and checking that $F(u) = O(G(u))$. The query complexity is $1$. 


\section{Preliminaries}\label{sec:prelims}
A \defn{language} $\mathcal{L}\subseteq\{0, 1\}^*$ is a set of
binary strings, and a family of circuits $\{C_n\}_{n\in\mathbb{N}}$,
where for every $n$, $C_n$ takes as input a binary string of length
$n$ and outputs a single bit, \defn{decides} a language $\mathcal{L}$
if for every $n$ and every $x\in\{0, 1\}^n$, we have $x\in\mathcal{L}$
if and only if $C_n(x) = 1$. A language is a formal description of a
problem; if a circuit family decides a language, we can also think of
it as solving the corresponding problem.


\subsection{Oracle-Aided Computation and Robustness}

In our protocols, we model the computation that the prover and verifier perform as an
oracle Boolean circuit. Informally, a Boolean circuit is a diagram that shows how to obtain
an output bit from a binary input string by applying some sequence of
OR ($\lor$), AND ($\land$), and NOT ($\neg$) operations, and an oracle Boolean circuit with access to \defn{oracle} $O: \{0, 1\}^*\to\{0, 1\}$ additionally can make calls to $O$.\footnote{We assume for simplicity that $O$ outputs a single bit. Our protocols can be readily extended to circuits with access to oracles that output multiple bits.} Here is a more formal definition:

\begin{definition}[(Oracle) Boolean circuit]

  For every $n\in\mathbb{N}$, an $n$-input, single-output Boolean
  circuit is a directed acyclic graph with $n$ \emph{sources}, i.e., vertices
  with no incoming edges, and one \emph{sink}, i.e., a vertex with no
  outgoing edges. Every non-source vertex is called a \emph{gate} and
  is labelled with one of $\lor$, $\land$, or $\neg$; $\lor$ and
  $\land$ gates have two incoming edges and $\neg$ gates have one
  incoming edge. For some input string $x\in\{0, 1\}^n$, the
  output of the $i$th source vertex is the $i$th bit of $x$,
  and the value of a gate is defined recursively as the result of
  applying the logical operation of the gate on the values of its \emph{children}, i.e., the vertices with an edge going into the
  gate. The \emph{output} of the circuit on $x$ is the output of the
  sink vertex.

  The \emph{size} of a circuit is the number of gates it contains, and the \emph{depth} is the length of the longest path from a source vertex to the sink vertex.

  An oracle Boolean circuit with respect to oracle
  $O: \{0, 1\}^*\to\{0, 1\}$ additionally has oracle gates,
  which have an arbitrary number of incoming edges. We view the input
  to an oracle gate as a string given by the values of its
  children. The value of an oracle gate is the value of the oracle
  applied to its input string.
  
\end{definition}

We say that an oracle circuit makes \defn{adaptive} oracle queries if
it contains two oracle gates that are connected by a path. This means
that there is some oracle query that depends on the answer to another
oracle query. An oracle circuit that does not make adaptive oracle
queries makes \defn{nonadaptive} oracle queries. A \defn{$d$-adaptive}
oracle circuit is one where the oracle gates are divided into $d$
levels, and an oracle gate in level $i$ can only be connected to
oracle gates in levels $i+1, \ldots, d$, i.e., queries can depend only
on answers corresponding to oracle gates in smaller levels. 
We can view any oracle circuit as $d$-adaptive: a circuit that makes
nonadaptive oracle queries can be viewed as a circuit with one level
of oracle gates, and a circuit that makes adaptive queries as a
circuit with multiple levels. 

Now we formally define what it means for an oracle circuit to be
robust. Recall that what we want from this definition is that the
output of a robust circuit should not change if at most a small
fraction of the answers to its oracle queries are incorrect. Consider
an $n$-input circuit $C$ with access to oracle $O$; we denote this by
$C^O$. On input $x\in\{0, 1\}^n$, $C^O$ makes a sequence of oracle
queries and gets the corresponding answers, which we can denote by the
string $Q = Q(x) = ((x_1, y_1), \ldots, (x_\ell, y_\ell))$,
where for each $i$, $x_i$ is the $i$th oracle query and $y_i = O(x_i)$
the $i$th oracle answer. We call $Q$ the \defn{true query-answer}
string. One could also imagine substituting $Q$ with another
length-$\ell$ string
$Q' = ((x_1', y_1'), \ldots, (x_\ell', y_\ell'))$, where
for all $i$, $x_i'$ is a binary string and $y_i'$ is a bit, though it
may not be the case that $x_i'$ is computed correctly from $x$ and previous oracle answers, or that $y_i' = O(x_i')$---we call a string of this
form a \defn{query-answer} string.  For a query-answer string $S$, let us use $S^{\text{in}}$ (resp. $S^\text{out}$) to denote the string of queries (resp. answers) in $S$. Then we could run the computation of $C$
on $x$ with ``help'' from $Q'^{\text{out}}$ instead, i.e., replacing the oracle
answers as given by $Q^{\text{out}}$ with those given by $Q'^{\text{out}}$. We
will denote this circuit by $C_{Q'}$. Note that
$C^O(x) = C_{Q}(x)$.

We first define what it means for a true query-answer string and a query-answer
string to be ``close,'' which we will then use to define
robustness. Let $Q$ be the true query-answer string for an oracle circuit
$C^O$ on input $x$, and write $Q = (Q_1, \ldots, Q_d)$, where $Q_i$
consists of all query-answer pairs in $Q$ for oracle gates in the $i$th
adaptivity level. We can
decompose $C^O$ level by level into circuits $C_0,\ldots,C_d$ without oracle gates as
follows: $C_0$ maps $x\mapsto Q_1^{\text{in}}$, for $i\in[d-1]$, $C_i$ maps
$(x,Q_1, \ldots, Q_i)\mapsto Q_{i+1}^{\text{in}}$, and $C_d$ maps
$(x,Q_1, \ldots, Q_d)\mapsto C^O(x)$.

\begin{definition}[Closeness]\label{defn:closeness}

  Consider the following ``$\eps$-closeness'' algorithm: given a
  $d$-adaptive oracle circuit $C^O$ with $\ell$ oracle gates, an input $x$,
  and the true query-answer string $Q = (Q_1, \ldots, Q_d)$, we set
  $T_1 = Q_1$ and modify some positions in $T_1$, obtaining a string
  $Q_1'$. Then for $i = 2, \ldots, d$, we set
  $T_{i}^{\text{in}} = C_{i-1}(x, Q_1', \ldots, Q_{i-1}')$ and $T_i = (T_{i}^{\text{in}}, O(T_i^{\text{in}}))$, where $O(T_{i}^{\text{in}})$ is obtained by applying $O$ to each query in $T_{i}^{\text{in}}$, and modify some positions
  in $T_i$, obtaining a string $Q_i'$. If we can obtain
  $Q' = (Q_1', \ldots, Q_d')$ via this algorithm by modifying at most
  $\eps\ell$ positions (across all of the $T_i$'s), then we say that
  $Q'$ is $\eps$-close to $Q$.\footnote{Note that, when the circuit makes
    nonadaptive queries, this definition of closeness is exactly a
    fractional Hamming distance definition, i.e., $Q$ and $Q'$ are
    $\eps$-close if we can turn $Q$ into $Q'$ by changing at most
    $\eps\ell$ of its positions. We use this more involved definition
    because in circuits that make adaptive queries, changing the
    output of an oracle gate can affect what the ``correct'' input is to
    another oracle gate.}
  
\end{definition}  

\begin{definition}[Robustness]\label{defn:robustness}

  We say that $C^O$ is $\eps$-robust if, for any length-$n$ binary
  string $x$, $C_{Q'}(x) = C^O(x)$ for any $Q'$ that is
  $\eps$-close to the true query-answer string $Q$.
  
\end{definition}  

\subsection{Succinct Descriptions of Sets and Functions}
\label{sec:desc}

We next define a notion of succinct representation of circuits. Loosely speaking, a function $f: \bin^n \to \bin$ has a succinct representation if there is a short string $\lrag f$ of poly-logarithmic length that describes $f$. That is, $\lrag f$ can be expanded to a full description of $f$. The actual technical definition is slightly more involved and in particular requires that the full description of $f$ be a logarithmic-depth (i.e. $\NC^1$) circuit:

\begin{definition}[Succinct Description of Functions]\label{def:funcdesc}
  We say that a function $f : \bin^n \to \bin$ of size $s$ has a succinct
  description if there exists a string $\lrag f$ of length $\polylog(n)$
  and a logspace Turing machine $M$ (of constant size, independent of
  $n$) such that on input $1^n$, the machine $M$ outputs a full
  description of an $\NC^1$ circuit $C$ such that for every $x \in \bin^n$
  it holds that $C(\lrag f, x) = f(x)$. We refer to $\lrag f$ as the
  succinct description of $f$.
\end{definition}

We also define a notion of succinct representation for sets $S \subseteq [k]$.
Roughly speaking this means that the set can be described by a string of
length $\polylog(k)$. The formal definition is somewhat more involved:

\begin{definition}[Succinct Description of Sets]\label{def:setdesc}
  We say that a set $S \subseteq [k]$ of size $s$ has a succinct
  description if there exists a string $\lrag S$ of length $\polylog(k)$
  and a logspace Turing machine $M$ such that on input $1^k$, the machine
  $M$ outputs a full description of a depth $\polylog(k)$ and size
  $\poly(s, \log k)$ circuit (of constant fan-in) that on input $\lrag S$
  outputs all the elements of $S$ as a list (of length $s \cdot \log(k)$).
\end{definition}

We emphasize that the size of the circuit that $M$ outputs is
proportional to the actual size of the set $S$, rather than the universe
size $k$.

\subsection{Interactive Proofs}

An interactive proof is an interactive protocol between a
weak verifier algorithm $\mathcal{V}$ and a powerful prover algorithm
$\mathcal{P}$, where $\mathcal{P}$ tries to convince $\mathcal{V}$ of
a statement of the form ``$x\in\mathcal{L}$.''

\begin{definition}[Interactive proof system]\label{defn:ip}

  An interactive proof for a language $\mathcal{L}$ is an interactive
  protocol between a probabilistic polynomial-time verifier algorithm
  $\mathcal{V}$ and a computationally unbounded prover algorithm
  $\mathcal{P}$. On common input $x$, $\mathcal{V}$ and $\mathcal{P}$
  back-and-forth exchange messages in a number of rounds. In each
  round, $\mathcal{V}$ sends $\mathcal{P}$ a message and then
  $\mathcal{P}$ sends $\mathcal{V}$ a message. Both $\mathcal{V}$'s
  and $\mathcal{P}$'s messages can depend on $x$ and any prior
  messages, and $\mathcal{V}$'s messages can additionally depend on
  $\mathcal{V}$'s random bits $r$. At the end of this interaction,
  their messages form a transcript
  $t = (\mathcal{V}(r), \mathcal{P})(x)$, and based on $t$, $r$, and
  $x$, $\mathcal{V}$ decides to accept or reject. The protocol
  satisfies completeness and soundness properties, namely,

  \begin{itemize}

  \item (Completeness.) For every $x\in\mathcal{L}$, there exists an honest prover strategy $\mathcal{P}$ such that $\text{Pr}_r[V(x, t, r) = 1] = 1$, where $t = (\mathcal{V}(r), \mathcal{P})(x)$.

  \item ($\delta$-Soundness.) For every $x\notin\mathcal{L}$ and for every
    (computationally unbounded) prover algorithm $\widetilde{P}$,
    $\text{Pr}_r[V(x, \widetilde{t}, r) = 1]\leq\delta$, where
    $\widetilde{t} = (\mathcal{V}(r), \mathcal{\widetilde{P}})(x)$.
    
  \end{itemize}
  
\end{definition}

The \emph{parameters} of interactive proofs that we are interested in include the prover's running time, the verifier's running time, the number of rounds (i.e., the round complexity) and the total number of bits communicated (i.e., the communication complexity).

In a standard interactive proof, the verifier is constrained to run in
time polynomial in $|x| = n$ and the prover is computationally
unbounded. A \defn{doubly-efficient} interactive proof is one where
the honest prover is constrained to run in polynomial time (though
soundness still holds against computationally unbounded dishonest
provers), and the verifier is even more efficient, running in
near-linear time. In our setting of doubly-efficient interactive
proofs for oracle computation, we additionally require that the
verifier makes only a \emph{sublinear} number of queries to the oracle.

We will make use of a construction of doubly-efficient interactive
proofs due to Goldwasser, Kalai, and Rothblum:

\begin{theorem}[\cite{goldwasser2015delegating}]\label{deIP}

  Let $\mathcal{L}$ be a language that has logspace-uniform Boolean
  circuits of depth $D = D(n)$ and size $S = S(n)$. For any soundness parameter $\sigma \in \NN$, there is a
  doubly-efficient interactive proof for $\mathcal{L}$
  with the following parameters:

  \begin{itemize}

  \item soundness error $2^{-\sigma}$;
  \item communication complexity $O(\sigma\cdot D\cdot\polylog(S))$;
  \item round complexity $O(D\cdot\polylog(S))$;
  \item verifier running time $O(\sigma\cdot(n + D\cdot\polylog(S)))$; and
  \item prover running time $\poly(S, \sigma)$.

  \end{itemize}

\end{theorem}

\begin{proof}
  The protocol of \cite{goldwasser2015delegating} achieves the stated parameters with constant soundness error. Applying $\sigma$-fold parallel repetition reduces the soundness error to $2^{-\sigma}$, while preserving the round complexity and increasing the communication complexity and the prover's and verifier's running times by a factor of $\sigma$.
\end{proof}

\paragraph{Interactive arguments.} We also consider interactive \defn{argument systems}, which are defined the same as
interactive proof systems except that the soundness guarantee is only
required to hold against provers that run in polynomial time. The soundness guarantee that a proof system is required to satisfy is called \defn{statistical soundness}; argument systems are only required to satisfy \defn{computational soundness}. Unlike statistical soundness, with computational soundness, a malicious prover $\MProver$'s success probability (i.e., the probability that $\Verifier$ outputs $1$) depends on the running time of $\MProver$. We give a formal definition in \Cref{sec:prelims:AoK}.




  

\subsection{Interactive Proofs of Proximity}\label{sec:IPP}

Our proof systems for robust circuits make key use of an
interactive proof of proximity (IPP). Loosely speaking, an IPP is a
doubly-efficient interactive proof where the verifier is extremely
efficient, running in only sublinear time, and the soundness
requirement is weakened so that the verifier is only required to
reject inputs that are ``far'' from the language with high
probability, for, e.g., a fractional Hamming distance notion of
distance. We will deal with IPPs for \defn{pair languages}, where the
input to the verifier is a pair $(x, Q)$ consisting of an
\defn{explicit input} $x$ that the verifier has direct access to and
an \defn{implicit input} $Q$ that the verifier has query access to, i.e., the verifier treats $Q$ as an oracle. We
say that $(x, Q)$ is $\eps$-\defn{Hamming-far} from pair language
$\mathcal{L}$ for \defn{proximity parameter} $\eps\in (0, 1]$ if $Q$
differs in at least an $\eps$ fraction of its positions from every
$Q'$ such that $(x, Q')\in\mathcal{L}$, and that it is
$\eps$\defn{-Hamming-close} otherwise.

We denote the random variable of the outcome of an IPP as $\lrag {\Prover(x), \Verifier^Q(x)}$, where the randomness is over the randomness of $\Prover$ and $\Verifier$.

\begin{definition}[Interactive proof of proximity
  \cite{rothblum2013interactive, rothblum2020batch}]\label{def:ipp}

  An interactive proof of proximity for a pair language $\mathcal{L}$
  is an interactive protocol between a probabilistic sublinear-time
  verifier algorithm $\mathcal{V}$ and a polynomial-time prover
  algorithm $\mathcal{P}$. Both $\mathcal{V}$ and $\mathcal{P}$ have
  access to explicit input $x$ and a proximity parameter $\eps$, and
  $\mathcal{V}$ has query access to implicit input $Q$ while
  $\mathcal{P}$ has direct access to $Q$. The protocol satisfies
  completeness and a relaxed notion of soundness, namely

  \begin{itemize}

  \item (Completeness.) For every $(x, Q)\in\mathcal{L}$ and proximity
    parameter $\eps\in (0, 1]$, there exists an honest prover strategy $\mathcal{P}$ such that
    $\Pr[\lrag {\Prover(x), \Verifier^Q(x)} = 1] = 1$.

  \item ($\delta$-Soundness.) For every $(x, Q)$ that is $\eps$-Hamming-far
    from $\mathcal{L}$, and for every (computationally unbounded)
    prover algorithm $\widetilde{\mathcal{P}}$, we have
    $\text{Pr}_r[\lrag {\widetilde{\Prover}(x), \Verifier^Q(x)} = 1]\le\delta$.
    
  \end{itemize}  
  
\end{definition}

\begin{remark}[Input Encoding]
  In IPP, the input can be treated as a binary string, and it can also be regarded as an element in some finite field $\mathbb{F}$. In the latter case, the communication and verifier's running time also grow with $\log |\mathbb{F}|$, assuming addition and multiplication are efficient in that field.
\end{remark}

The best-known IPP construction is given by \citet{rothblum2020batch}.

\begin{theorem}[\cite{rothblum2020batch}]\label{rothblumx2ipp}

  Let $\eps = \eps(m)\in (0, 1]$ be a proximity parameter, let
  $\sigma \in \NN$ be a soundness parameter, and let
  $\mathcal{L}$ be a pair language that is computable by
  logspace-uniform Boolean circuits of depth $D = D(m)\ge\log(m)$ and
  size $S = S(m)\ge m$ with fan-in $2$, where we use $n$ for the
  length of the explicit input and $m$ for the length of the implicit
  input. Then there is a public-coin IPP with $\eps$ proximity for
  $\mathcal{L}$ with the following parameters:

  \begin{itemize}

  \item soundness error $2^{-\sigma}$;
  \item query complexity $q = O(\sigma/\eps)$;
  \item communication complexity $cc = O(\sigma\cdot\eps\cdot m\cdot D\cdot\polylog(S))$;
  \item round complexity $O(D\cdot\polylog(S))$;
  \item verifier running time $O(\sigma\cdot(\eps \cdot m \cdot D \cdot \polylog(S)+(1/\eps)\cdot\polylog(S) + n \log S))$; and
  \item prover running time $\poly(S,\sigma)$.
    
  \end{itemize}

  Furthermore, the verifier can make all of his queries to the implicit input at the end of the interaction. More formally, at the end of the interaction either the verifier rejects or in time $O(\sigma\cdot\eps \cdot m \cdot D \cdot \polylog(S))$ it outputs a succinct description $\lrag Q$ of a set $Q \subseteq [m]$ of size $q$ and a succinct description $\lrag \phi$ of a predicate $\phi: \bin^q \to \bin$ so that its decision predicate given an implicit input $x^{\mathsf{imp}}$ is equal to $\phi(x^{\mathsf{imp}}_Q)$. Furthermore, the predicate can be evaluated in time $q \cdot \polylog(m)$.

\end{theorem}

\begin{proof}
  The theorem follows by composing Theorem 5.6 and Theorem 6.1 in \cite{rothblum2020batch} and applying $\sigma$-fold parallel repetition to obtain soundness error $2^{-\sigma}$. Parallel repetition preserves the round complexity, while increasing the query and communication complexities and the running times by a factor of $O(\sigma)$.
\end{proof}

\subsection{Unique-Decoding Checksums}\label{sec:prelim:cksum}

Let $\rho \in \NN$ be a \emph{deviation radius} and $\cksumT \in \NN$ be a checksum-length parameter.
We use syndromes of linear error-correcting codes as checksums.
\begin{definition}[Unique-Decoding Checksums]
    \label{def:unique-decoding-checksum}
    Let $k,\rho,\cksumT \in \NN$ and let $\FF$ be a field.
    A function $\cksum_\rho:\bin^k \to \bin^\cksumT$ is a \emph{$\rho$-unique decoding checksum function}
    if for any $\bm m \in \bin^k$,
    and for any $\bm m', \bm m'' \in \bin^k$ that are both $\rho$-close to $\bm m$ and $\bm m' \neq \bm m''$,
    $\cksum_\rho(\bm m') \neq \cksum_\rho(\bm m'')$.
\end{definition}
The term \emph{unique decoding} refers to the fact that if we know $\bm m'$ is $\rho$-close to some (fixed) $\bm m$,
then we can uniquely determine $\bm m'$ given $\cksum_\rho(\bm m')$.

\begin{proposition}[Linear-code checksums]
    Let $\mathcal{C}\subseteq\FF^k$ be a linear code of minimum distance $D$ and codimension $\cksumT$, and $H\in\FF^{\cksumT\times k}$ be a parity-check matrix for $\mathcal{C}$. 
    Assume that $|\FF| = 2^{s}$ and $\mathsf{Enc}: \bin^s \to \FF, \mathsf{Dec}: \FF \to \bin^s$ are linear maps converting between a binary string and the corresponding field element. 
    For any $\rho\in\NN$ satisfying $2\rho < D$,
    the syndrome map
    \[
        \cksum_\rho: \bin^{ks} \to \bin^{s\cksumT}, \qquad
        \cksum_\rho(\bm m) \coloneqq \mathsf{Dec} (H \mathsf{Enc}(\bm m)),
    \]
    where $\mathsf{Enc}$ is applied blockwise and $\mathsf{Dec}$ coordinate-wise,
    is a $\rho$-unique decoding checksum function.
    Moreover, with $\tO$ hiding polylogarithmic factors, $\cksum_\rho$ can be evaluated in $\tO(k\cksumT s)$.
\end{proposition}
\begin{proof}
    Fix $\bm m\in\bin^{ks}$ and let $\bm m',\bm m''\in\bin^{ks}$ be distinct vectors that are both $\rho$-close to $\bm m$.
    Then $\bm z=\bm m'-\bm m''$ is nonzero and has Hamming weight at most $2\rho$.
    If $\cksum_\rho(\bm m')=\cksum_\rho(\bm m'')$,
    then $\cksum_\rho(\bm z) = \bm 0$ by linearity,
    so $\mathsf{Enc}(\bm z)\in\ker(H)=\mathcal{C}$.
    This is a nonzero codeword of Hamming weight at most $2\rho<D$,
    contradicting the minimum distance of $\mathcal{C}$.
    Hence $\cksum_\rho(\bm m')\neq \cksum_\rho(\bm m'')$.
\end{proof}

\begin{lemma}[Reed--Solomon checksums]\label{lem:unique-decoding}
    Let $k,\rho \in \NN$ and let $\FF$ be a finite field with $\abs{\FF} = 2^s \ge k$.
    Set $\cksumT \coloneqq 2\rho$.
    Then a $\rho$-unique decoding checksum function $\cksum_\rho: \bin^{k} \to \bin^{s \cksumT}$ exists.
    With $\tO$ hiding $\polylog(k)$ factors,
    $\cksum_\rho: \bin^k \to \bin^{s \cksumT}$ can be evaluated in $\tO(k \cksumT s)$ time.
\end{lemma}
\begin{proof}
    If $2\rho < k$,
    let $H\in\FF^{2\rho\times k}$ be a parity-check matrix of a Reed--Solomon code over $\FF$ with block length $k$,
    dimension $k-2\rho$,
    and minimum distance $2\rho+1$.
    Define $\cksum_\rho(\bm m)\coloneqq \mathsf{Dec}(H \mathsf{Enc}(\bm m || \bm 0))$, where $\bm 0$ is padding with appropriate length.
    By the previous proposition, this is a $\rho$-unique decoding checksum function.

    If $2\rho\ge k$,
    let $\cksum_\rho$ be the identity map on $\bin^k$ padded with $s\cksumT-k$ zero coordinates.
    Then $\cksum_\rho$ is injective, so it is $\rho$-unique decoding.

    The evaluation bound follows by computing the $\cksumT$ linear forms defining $\cksum_\rho$.
\end{proof}


\subsection{Interactive Arguments and Arguments of Knowledge}\label{sec:prelims:AoK}

We consider interactive arguments and arguments of knowledge.  

Let $\mathcal{R} \subseteq \bin^* \times \bin^* \times \bin^*$ be a ternary relation. If $(\pp, x, w) \in \Rel$, we say that $\pp$ are the public parameters, $x$ is a statement and $w$ is a witness for $x$. Define the language $\mathcal{L}$ as $\Set{(\pp, x): \exists w \text{ s.t. } (\pp, x, w) \in \mathcal{R}}$.

\begin{definition}[Interactive argument system]
  Let $\ell \ge 0$ be an integer. A $(2\ell + 1)$-message public-coin argument system $\Pi = (\Setup, \calP, \calV)$ for a relation $\Rel$ consists of a PPT (probabilistic polynomial time) algorithm $\Setup$ and a $(2\ell + 1)$-message protocol between an interactive PPT prover $\calP$ and an interactive PPT verifier $\calV$ associated with a tuple $(X, W, (Z_{i-1}, C_i)_{i \in [\ell]}, Z_\ell)$, with the following properties:

\begin{itemize}
    \item The $\Setup$ algorithm takes as input the security parameter $1^\lambda$ and outputs some public parameters $\pp$.
    \item Both $\calP$ and $\calV$ receive as input the public parameters $\pp$ and a statement $x_0 = x \in X$. The prover $\calP$ additionally receives a witness $w_0 = w \in W$.
    \item The public parameters $\pp$, the statement $x_0$, and the $2\ell + 1$ messages sent by $\calP$ and $\calV$ in the protocol, are collectively called a transcript, labelled as
    \[
    (\pp, x_0, z_0, c_1, \dots, z_{\ell-1}, c_\ell, z_\ell),
    \]
    where $z_i \in Z_i$ is sent by $\calP$ and $c_i \in C_i$ is sent by $\calV$.
    \item The challenges $c_i$ are sampled by $\calV$ uniformly at random from $C_i$.
\end{itemize}

\noindent A transcript $(\pp, x_0, z_0, c_1, \dots, z_{\ell-1}, c_\ell, z_\ell)$ is said to be accepting for $\Pi$ if $\calV(\pp, x_0, z_0, c_1, \dots, z_{\ell-1}, c_\ell, z_\ell) = 1$ holds.
\end{definition}

Now we define completeness, soundness and knowledge soundness.

\begin{definition}[Completeness]
  An argument system $\Pi = (\Setup, \calP, \calV)$ for the relation $\Rel$ has statistical completeness with correctness error $\epsilon$ if for all adversaries $\calA$,
\[
\Pr \left[ b = 0 \land (\pp, x, w) \in \Rel \; \middle| \; 
\begin{array}{c}
\pp \leftarrow \Setup(1^\lambda) \\
(x, w) \leftarrow \calA(\pp) \\
(\tr, b) \leftarrow \langle \calP(\pp, x, w), \calV(\pp, x) \rangle
\end{array}
\right] \le \epsilon(\lambda).
\]
Furthermore, we say that $\Pi$ satisfies perfect completeness if $\epsilon = 0$.

\end{definition}

\begin{definition}[Soundness]
  An argument system $\Pi = (\Setup, \calP, \calV)$ for the relation $\Rel$ has computational soundness if for all stateful PPT adversary $\calP^*$, there exists a function $\delta$ negligible in $\lambda$, such that:
\[
\Pr \left[ b = 1 \land (\pp, x) \notin \lang \; \middle| \; 
\begin{array}{c}
\pp \leftarrow \Setup(1^\lambda) \\
(x, \mathsf{st}) \leftarrow \calP^*(\pp) \\
(\tr, b) \leftarrow \langle \calP^*(\pp, x, \mathsf{st}), \calV(\pp, x) \rangle
\end{array}
\right] \le \delta(\lambda).
\]
\end{definition}

\begin{definition}[Knowledge soundness]
  An argument system $\Pi = (\Setup, \calP, \calV)$ is knowledge sound with knowledge error $\kappa$ for the relation $\Rel^*$ if there exists an expected PPT extractor $\calE$ such that for any stateful PPT adversary $\calP^*$:
\[
\Pr \left[ b = 1 \land (\pp, x, w) \notin \Rel^* \; \middle| \; 
\begin{array}{c}
\pp \leftarrow \Setup(1^\lambda) \\
(x, \mathsf{st}) \leftarrow \calP^*(\pp) \\
(\tr, b) \leftarrow \langle \calP^*(\pp, x, \mathsf{st}), \calV(\pp, x) \rangle \\
w \leftarrow \calE_{\calP^*}(\pp, x)
\end{array}
\right] \le \kappa(\lambda).
\]
Here, the extractor $\calE$ has a black-box oracle access to the (malicious) prover $\calP^*$ and can rewind it to any point in the interaction.
\end{definition}

The classical Kilian's protocol \cite{kilian1992note} is an argument of knowledge for any NP language, i.e., $\Rel = (\bot, x, w)$ where $(x, w)$ is a valid instance-witness pair for some NP relation. Furthermore, Kilian's protocol has the following succinctness property: the communication complexity is $\poly(\secp)$, and the running time of the verifier is $\poly(\secp) \cdot |x|$.

\begin{theorem}\label{thm:kilian}
  Assuming the existence of collision-resistant hash functions, there exists a succinct argument of knowledge for any NP relation $\Rel$ with knowledge error negligible in the security parameter $\lambda$.
\end{theorem}

\subsection{Hash Trees}

Our argument system for robust oracle circuits makes use of a hash tree, which is a verifiable hash function that supports local openings (i.e., it allows a prover to commit to a list of elements and later open any individual element) that can be instantiated from any collision-resistant hash function~\cite{C:Merkle89a}.

\begin{definition}\label{def:hash_tree}
  A hash tree consists of a tuple of algorithms $(\htgen, \hthash, \htprove, \htver)$ with the following syntax:
\begin{itemize}
    \item $\htgen(\secparam, S) \to \hk$. This is a randomized algorithm that takes as input a security parameter $\sec$, and a space bound $S$, and outputs a hash key $\hk$. We implicitly assume that $\hk$ includes $\secparam, S$.

    \item $\hthash(\hk, D) \to \rt$. This is a deterministic algorithm that on input a hash key $\hk$ and a database $D \in \bin^S$ outputs a root $\rt$. 
    
    \item $\htprove(\hk, D, i) \to \rho$. This is a deterministic algorithm that on input a hash key $\hk$, a database $D \in \bin^S$, an index $i \in [S]$, outputs an opening proof $\rho$.
    When a set $I \subseteq [S]$ is given as input, we interpret this as providing an opening proof for every $i \in I$. 
    
    \item $\htver(\hk, \rt, i, v,\rho) \to b$. This is a deterministic algorithm that on input a hash key $\hk$, a hash tree root $\rt$, an index $i \in [S]$, a bit $v \in \bin$, and an opening proof $\rho$, outputs a bit $b \in \bin$ indicating whether to accept or reject the opening proof.
    When a set $I \subseteq [S]$ and values $v \in \bin^{|I|}$ are given as input, we interpret this as verifying, for every $j \in [|I|]$, the proof for the value $v[j]$ at the $j$th smallest index in $I$. 
\end{itemize}

\noindent We require $(\htgen, \hthash, \htprove, \htver)$ to satisfy the following properties:
\begin{itemize}
    \item \textbf{Opening completeness.} For any $\sec,S \in \bbN$, $i \in [S]$, database $D \in \bin^S$, it holds that 
    $$\Pr\left[
    \begin{array}{l}
    \htver(\hk, \rt, i, D[i], \rho) = 1 
    \end{array}
    :
    \begin{array}{l}
    \hk \gets \htgen(\secparam, S) \\
    \rt = \hthash(\hk, D) \\
    \rho = \htprove(\hk, D, i)
    \end{array}
    \right] = 1.$$
    \item \textbf{Efficiency.}
    In the opening completeness experiment above for $S \le 2^\sec$, $\htgen$ and $\htver$ run in $\poly(\secp)$ time, and $\hthash$ and $\htprove$ run in $S \cdot \poly(\secp)$ time. Also we require $|\hk| \leq \poly(\secp)$ and $|\rho| \leq \poly(\secp)$.

    \item \textbf{Binding.} For any non-uniform polynomial-time algorithm $A = \{A_\sec\}_{\sec\in\bbN}$ and polynomial $S$, there exists a negligible function $\mu$ such that for all $\sec\in\bbN$, it holds that 
    $$\Pr\left[
    \begin{array}{l}
    v \neq v' \\ 
    \wedge \; 1 = \htver(\hk, \rt, i, v, \rho) \\ 
    \wedge \; 1 = \htver(\hk, \rt, i, v', \rho') 
    \end{array}
    :
    \begin{array}{l}
    \hk \gets \htgen(\secparam, S(\sec)) \\
    (\rt, i, v, v', \rho, \rho') \gets A_\sec(\hk) 
    \end{array}
    \right] \le \mu(\sec).$$
\end{itemize}
\end{definition}

\begin{remark}
  For convenience, we may take the characters in the database $D$ as some finite set $\Sigma$ instead of just binary bits. The definition and properties of hash trees can be easily adapted to this setting, with a $\log |\Sigma|$ overhead on all efficiency parameters.
\end{remark}

\subsection{Polynomial Commitments}\label{sec:prelim:poly_commitment}

Our argument system for the setting where the oracle is low-degree uses a special commitment scheme called a polynomial commitment, which allows a prover to commit to a degree-bounded polynomial and later open evaluations of the polynomial at specific points, along with (interactive) proofs that the evaluations are correct. 

\begin{definition}
  A (non-interactive) commitment scheme over $\mathcal{M}$ is a tuple of polynomial-time probabilistic algorithms $\algo{CM} = (\algo{Setup}, \algo{Commit}, \algo{Open})$ with the following syntax.
\begin{itemize}
    \item[--] $\algo{Setup}(1^\lambda, d) \to \algo{pp}$: Sample public parameters given a security parameter $\lambda$ and message length $d$.
    \item[--] $\algo{Commit}(\algo{pp}, f) \to (C, \algo{st})$: Use the public parameters $\algo{pp}$ to compute a commitment $C$ to a message $f \in \mathcal{M}$ and an auxiliary state $\algo{st}$.
    \item[--] $\algo{Open}(\algo{pp}, C, f, \algo{st}) \to b$: Takes public parameters $\algo{pp}$, a commitment $C$, a message $f \in \mathcal{M}$, and an auxiliary state $\algo{st}$, and outputs a bit $b$ indicating whether $C$ is a valid commitment to $f$ under $\algo{pp}$.
\end{itemize}

\noindent We require commitment schemes to satisfy the following completeness and (relaxed) binding properties.
\end{definition}

\begin{definition}[Completeness]
  A commitment scheme $\algo{CM} = (\algo{Setup}, \algo{Commit}, \algo{Open})$ satisfies completeness if for all $\lambda, d \in \mathbb{N}$, and for every $f \in \mathcal{M}$
\[
\Pr \left[ \algo{Open}(\algo{pp}, C, f, \algo{st}) = 1 \; \middle| \; 
\begin{array}{c}
\algo{pp} \leftarrow \algo{Setup}(1^\lambda, d) \\
(C, \algo{st}) \leftarrow \algo{Commit}(\algo{pp}, f)
\end{array}
\right] \ge 1 - \negl(\lambda).
\]
\end{definition}

\begin{definition}[Binding]
  A commitment scheme $\algo{CM} = (\algo{Setup}, \algo{Commit}, \algo{Open})$ satisfies relaxed binding if for every PPT adversary $\mathcal{A}$,
\[
\Pr \left[ 
\begin{gathered}
f \neq f' \text{ with } f, f' \in \mathcal{M} \\
\land \\
\algo{Open}(\algo{pp}, C, f, \algo{st}) = \algo{Open}(\algo{pp}, C, f', \algo{st}') = 1
\end{gathered} 
\; \middle| \; 
\begin{array}{c}
\algo{pp} \leftarrow \algo{Setup}(1^\lambda, d) \\
(C, (f, \algo{st}), (f', \algo{st}')) \leftarrow \mathcal{A}(\algo{pp})
\end{array}
\right] \le \negl(\lambda).
\]
\end{definition}


Now we define an extractable polynomial commitment.

\begin{definition}[Extractable polynomial commitment scheme]
  Let $\text{PC} = (\Setup_{\text{CM}}, \Commit, \Open, \Setup_{\text{IP}}, \calP, \calV)$ be a tuple of algorithms. $\text{PC}$ is an extractable polynomial commitment scheme for function class $\calF$ if
\begin{itemize}
    \item $(\Setup_{\text{CM}}, \Commit, \Open)$ is a commitment scheme over the function class
    \[
    \calM := \calF.
    \]
    \item $(\Setup_{\text{IP}}, \calP, \calV)$ is an argument system for the relation
    \[
    (\pp, (\pp_{\text{CM}}, C, \bx, \bu), (f, \st)) \in \Rel \quad \Longleftrightarrow \quad \Open(\pp_{\text{CM}}, C, f, \st) = 1 \land f(\bx) = \bu.
    \]
\end{itemize}
\noindent The class of functions $\calF$ supported by a polynomial commitment scheme will be a set of polynomials. We say that the polynomial commitment scheme satisfies completeness and knowledge soundness if $(\Setup_{\text{IP}}, \calP, \calV)$ is complete and knowledge sound respectively. 
\end{definition}

\begin{theorem}[\cite{C:CMNW24}]\label{thm:prelim_poly_commitment}
  Assuming the standard (Module)-SIS assumption, there exists an extractable polynomial commitment scheme with $\poly(\secp, \log L)$ communication and verification times, where $L$ is the degree bound of the committed polynomial. Furthermore, the number of rounds of the local opening argument system is $O(\log L)$.
\end{theorem}

\section{Doubly-Efficient Interactive Proofs for Robust Oracle Circuits}\label{sec:robust}
Here we present our doubly-efficient single-prover
interactive proof and argument systems for robust oracle circuits. We
first present a proof system for robust circuits that make nonadaptive
oracle queries in \cref{sec:non-adaptive} below. In
\cref{sec:adaptive}, we extend this to a proof system for general robust circuits (that may make adaptive queries). In \cref{sec:arg}, we give an argument system for general robust circuits.

\subsection{A Proof System for Circuits Making Nonadaptive Oracle Queries}\label{sec:non-adaptive}

\begin{theorem}\label{thm:proof-nonadaptive}

  Let $\eps = \eps(n)\in (0, 1]$, $\sigma \in \NN$, and let
  $\mathcal{L}$ be a language decidable by $\eps$-robust logspace-uniform
  oracle Boolean circuits with $\ell = \ell(n)$ nonadaptive oracle gates, depth
  $D = D(n)\ge\log(n)$, and size $S = S(n)\ge n$, where all gates
  except possibly the oracle gates have fan-in $2$. We assume that the query length is bounded by $m = \poly(n)$; and
  $\eps \ge 1/\ell$, as otherwise $\eps$-robustness holds vacuously
  (see \cref{defn:robustness}). Then there is a
  doubly-efficient interactive proof system for $\mathcal{L}$ with the
  following parameters:

  \begin{itemize}

  \item soundness error $2^{-\sigma}$;
  \item query complexity $q = O(\sigma/\eps)$;
  \item communication complexity $cc = O(\sigma \cdot m \cdot \eps\cdot\ell \cdot D\cdot\polylog(S) + (\sigma / \eps) \cdot m \cdot \polylog(\ell))$;
  \item round complexity $D\cdot\polylog(S)$;
  \item verifier running time $O(\eps\ell \cdot D +1/\eps)\cdot \sigma \cdot m \cdot\polylog(S) + O(\sigma n\log S)$; and
  \item prover running time $\poly(S, \sigma)$.
    
  \end{itemize}  

\end{theorem}

\begin{figure}[htbp]
\fbox{
\begin{minipage}{\linewidth}
  \RaggedRight
  \vspace{.3cm}
  \underline{\textbf{DEIP for robust, nonadaptive setting}}\\
  \vspace{.3cm}
  \textbf{Input:} The prover $\mathcal{P}$ and verifier
  $\mathcal{V}$ receive $x\in\{0,1\}^n$, a succinct
  description of an $\eps$-robust logspace-uniform oracle circuit $C$ with $\ell$ nonadaptive oracle gates, depth $D$, and size $S$, and oracle access to $O:\{0,1\}^*\to\{0,1\}$. Note that $\eps$ is a fixed parameter, while $\ell$, $D$, and $S$ are parameters of the input. \\
  \vspace{.2cm}
  \textbf{Additional parameters:} We let $\sigma$ denote the soundness parameter.\\ 
  \vspace{.2cm}
  \textbf{Notations:} Enumerate the oracle gates by
  $1,\ldots,\ell$. Let $Q = Q(x) = ((x_1, y_1),\ldots,(x_\ell, y_\ell))$
denote the true query-answer string and $Q^{\text{in}}$ (resp. $Q^{\text{out}}$) denote the strings of queries (resp. answers) in $Q$. We split $C^O$ into two circuits
without oracle gates: $C_0$ computes $Q^{\text{in}}$ from $x$, and $C_1$ computes
$C^O(x)$ from $(x,Q)$.\\
\vspace{.2cm}
\textbf{Ingredients:} We make use of the following tools:
\begin{itemize}
  \item An IPP for the pair language
    $\mathcal{L}_\mathrm{acc} = \{(x,Q) : C_1(x,Q)=1\}$ with proximity
    parameter $\eps$ and soundness parameter $\sigma+2$ (instantiated by \Cref{rothblumx2ipp}).
  \item A DEIP for the language $\lang^*_{J, A^{\text{in}}}$ defined below (which asserts correctness of $C_0(x)$ on a subset of
    coordinates), with soundness parameter $\sigma+2$ (instantiated by \Cref{deIP}).
\end{itemize}
\textbf{The protocol:}
\begin{enumerate}[itemsep=.1cm]

\item The prover evaluates $C^O(x)$ and computes the true query-answer string $Q$. The prover and the verifier run the IPP for $\mathcal{L}_\mathrm{acc}$ with proximity
  parameter $\eps$ on input $(x, Q)$, where $x$ is the explicit input and $Q$ is the implicit input. At the end of this
  interaction, the IPP verifier outputs the description $\lrag J$ of indices $J\subseteq[\ell]$ that it wishes to query in $Q$ and the description $\lrag \phi$ of the predicate $\phi$. The verifier sends $\lrag J$ to the prover.

\item The prover sends $A\coloneq Q|_J$. The verifier finishes the IPP verification by checking $\phi(A) = 1$. The verifier also checks that the oracle calls in $A$ are correct: for all $i\in J$,
    it checks that $y_i = O(x_i)$. 
    
\item Letting $A^{\text{in}} = (x_i)_{i\in J}$ denote
    the string of queries in $A$, the prover and the verifier
    engage in a DEIP to prove that $C_0(x)|_{J} =
    A^{\text{in}}$. Formally, letting 
    \[
    \lang^*_{J, A^{\text{in}}} = \Set{x \in \bin^n: C_0(x)|_{J} = A^{\text{in}}},
    \]
    the prover and verifier run a DEIP to prove that $x$ is in $\lang^*_{J, A^{\text{in}}}$. 

\item The verifier accepts if the IPP and DEIP verifiers accept
    and all of its checks pass, and rejects otherwise.
  
\end{enumerate}  

\end{minipage}
}
\caption{DEIP for robust circuits making nonadaptive queries}\label{fig:fig-nonadaptive}
\end{figure}  

\begin{proof}

The protocol is given in Figure~\ref{fig:fig-nonadaptive}. We analyze its completeness, soundness, and efficiency. 

  \paragraph{Completeness.}
  If $C^O(x)=1$, then $(x, Q)\in \mathcal{L}_\mathrm{acc}$.
  The honest prover answers the verifier's query in step 2 with
  $A = Q|_J$. Then the verifier's oracle correctness checks pass, the DEIP claim $C_0(x)|_{J}=A^{\text{in}}$ holds, and by completeness of the IPP and DEIP the verifier accepts with probability 1.

  \paragraph{Soundness.}
  Assume that $x \notin \lang$ and fix any dishonest prover strategy $\MProver$. 
  Let $J$ be the set of indices queried by the IPP verifier in step 1, $\phi$ be the predicate the IPP verifier outputs, $Q$ be the true query-answer string (that is, the query-answer string obtained from an \emph{honest} evaluation of $C^O(x)$),
  and $\widetilde{A} = \Set{(\tilde{x}_i, \tilde{y}_i): i \in J}$ be the prover's response in step 2. Here $Q$ is a fixed string determined by $x$, while $J$, $\phi$, and $\widetilde{A}$ are random variables determined by the interaction $\lrag{\MProver^{O}(x), \Verifier^O(x)}$.

  We define the following events over this interaction:
  \begin{itemize}
    \item $\mathsf{Accept}$: the verifier accepts, i.e., $\lrag{\MProver^{O}(x), \Verifier^O(x)} = 1$.
    \item $\mathsf{Cons}$: the prover's response $\widetilde{A}$ is \emph{consistent}, i.e., $\widetilde{A} = Q|_J$. Technically, we define $\mathsf{Cons}$ as the following two holding simultaneously: oracle answers are correct and queries agree with the honest evaluation $C_0(x)$ on $J$:
    \[
      \mathsf{Cons} := \big(\forall i\in J,\ \tilde{y}_i = O(\tilde{x}_i)\big)\ \wedge\ \big(C_0(x)|_J = \widetilde{A}^{\text{in}}\big).
    \]
    \item $\mathsf{IPPAcc}$: the IPP predicate accepts the \emph{true} answers, i.e., $\phi(Q|_J) = 1$.
  \end{itemize}

  Our goal is to show $\Pr[\mathsf{Accept}] \leq 2^{-\sigma}$. The analysis relies on the following three claims, which we prove below.

  \begin{claim}\label{clm:nonadaptive-cons}
    $\Pr[\mathsf{Accept} \wedge \neg\mathsf{Cons}] \leq 2^{-(\sigma+2)}$.
  \end{claim}

  \begin{claim}\label{clm:nonadaptive-ipp}
    $\Pr[\mathsf{IPPAcc}] \leq 2^{-(\sigma+2)}$.
  \end{claim}

  \begin{claim}\label{clm:nonadaptive-implication}
    $\mathsf{Accept} \wedge \mathsf{Cons} \implies \mathsf{IPPAcc}$.
  \end{claim}

  Granting the claims, we first conclude the soundness analysis. Since
  \[
    \mathsf{Accept}
    \;\iff\;
    \big(\mathsf{Accept} \wedge \neg\mathsf{Cons}\big)
    \;\vee\;
    \big(\mathsf{Accept} \wedge \mathsf{Cons}\big),
  \]
  \cref{clm:nonadaptive-implication} gives
  \[
    \mathsf{Accept}
    \;\implies\;
    \big(\mathsf{Accept} \wedge \neg\mathsf{Cons}\big)
    \;\vee\;
    \mathsf{IPPAcc}.
  \]
  Therefore,
  \begin{align}
    \Pr[\mathsf{Accept}]
    &\leq \Pr[\mathsf{Accept} \wedge \neg\mathsf{Cons}] + \Pr[\mathsf{IPPAcc}] && \text{(union bound)} \\
    &\leq 2^{-(\sigma+2)} + 2^{-(\sigma+2)} && \text{(\cref{clm:nonadaptive-cons}, \cref{clm:nonadaptive-ipp})} \\
    &= 2^{-(\sigma+1)} \;\leq\; 2^{-\sigma}.
  \end{align}

  It remains to prove the three claims.

  \begin{proof}[Proof of \cref{clm:nonadaptive-cons}]
    The verifier checks every oracle answer in $\widetilde{A}$ directly: for each $i \in J$, it queries $O(\tilde{x}_i)$ and rejects unless $\tilde{y}_i = O(\tilde{x}_i)$. Therefore,
    \[
    \Pr[\mathsf{Accept} \wedge \neg (\forall i \in J,\ \tilde{y}_i = O(\tilde{x}_i))] = 0.
    \]
    The second conjunct of $\mathsf{Cons}$ is exactly the DEIP statement $x \in \lang^*_{J, \widetilde{A}^{\text{in}}}$. Since the verifier accepts only when the DEIP verifier accepts, by the $2^{-(\sigma+2)}$-soundness of the DEIP (\Cref{deIP}),
    \[
    \Pr[\mathsf{Accept} \wedge \neg \big(C_0(x)|_J = \widetilde{A}^{\text{in}}\big)] \leq 2^{-(\sigma+2)}.
    \]
    The claim follows by a union bound over the two conjuncts.
  \end{proof}

  \begin{proof}[Proof of \cref{clm:nonadaptive-ipp}]
    We first claim that $(x,Q)$ is $\eps$-Hamming-far from $\mathcal{L}_\mathrm{acc}$. Indeed, let $Q'$ be any query-answer sequence for input $x$ which differs from $Q$ on at most an $\eps$ fraction of indices. Since $x \notin \lang$, we have $C_1(x, Q) = C^O(x) = 0$, and by $\eps$-robustness of $C^O$, $C_1(x, Q') = C_1(x, Q) = 0$. Therefore, $(x, Q') \notin \mathcal{L}_\mathrm{acc}$.

    Recall that the IPP for $\mathcal{L}_\mathrm{acc}$ is instantiated with proximity parameter $\eps$ and soundness parameter $\sigma+2$ (\Cref{rothblumx2ipp}), and note that $\phi(Q|_J)$ is exactly the decision of the IPP verifier when its queries to the implicit input are answered according to the true sequence $Q$. Therefore, by the soundness guarantee of the IPP applied to the $\eps$-far input $(x, Q)$,
    \[
      \Pr[\mathsf{IPPAcc}] = \Pr[\phi(Q|_J) = 1] \leq 2^{-(\sigma + 2)},
    \]
    where the probability is over the IPP interaction in step 1.
  \end{proof}

  \begin{proof}[Proof of \cref{clm:nonadaptive-implication}]
    Suppose that $\mathsf{Accept}$ and $\mathsf{Cons}$ both hold. By definition, $C_0(x) = Q^{\text{in}}$ is the true query sequence obtained from an honest evaluation, so the second conjunct of $\mathsf{Cons}$ gives $\tilde{x}_i = x_i$ for all $i \in J$, and the first conjunct then gives $\tilde{y}_i = O(\tilde{x}_i) = O(x_i) = y_i$. Hence $\widetilde{A} = Q|_J$: the answers fed to the IPP predicate are exactly the answers obtained from query access to the true implicit input $Q$. Finally, since the verifier accepts, the verifier's check $\phi(\widetilde{A}) = 1$ passes, and therefore $\phi(Q|_J) = \phi(\widetilde{A}) = 1$, i.e., $\mathsf{IPPAcc}$ holds.
  \end{proof}

  \paragraph{Efficiency.}
  We obtain the claimed parameters by examining the efficiency in each step. Recall that both ingredients are instantiated with soundness parameter $\sigma + 2 = O(\sigma)$, and that $\eps\ell \ge 1$ by assumption.

  \begin{itemize}
    \item \textbf{Step 1 (IPP).}
    \begin{itemize}
      \item \textbf{Circuit proven.} The IPP of Step~1 is run on the pair language
        $\mathcal{L}_\mathrm{acc} = \{(x,Q) : C_1(x,Q)=1\}$, decided
        by the oracle-free circuit $C_1$ followed by a single output check.
        Since $C_1$ is obtained from $C$ by deleting the oracle gates and
        feeding their answers in as inputs, it has depth at most $D$ and size
        at most $S$; the explicit input $x$ has length $n$ and the implicit
        input $Q$ has $\ell$ positions.
      \item \textbf{Efficiency.} Instantiating \Cref{rothblumx2ipp} with proximity parameter $\eps$, soundness parameter $\sigma+2$, explicit-input length $n$, implicit-input length $\ell$, depth $D$, and size $S$ yields query complexity $q = O(\sigma/\eps)$, communication $O(\sigma\cdot\eps \cdot m\cdot\ell\cdot D\cdot\polylog(S))$, round complexity $O(D\cdot\polylog(S))$, and verifier running time $O(\sigma\cdot(\eps m \ell D\polylog(S) + (1/\eps)\polylog(S) + n\log S))$. At the end the verifier outputs the succinct description $\lrag J, \lrag \phi$ of the query set $J\subseteq[\ell]$ with $|J| = q = O(\sigma/\eps)$ and its decision predicate, and sends $\lrag J$ to the prover. Since $\lrag J$ is produced in time $O(\sigma\cdot\eps\cdot\ell\cdot D\cdot\polylog(S))$, sending it increases the communication by at most a constant factor. 
    \end{itemize}
    \item \textbf{Step 2 (answers and DEIP).}
    \begin{itemize}
      \item \textbf{Circuit proven.} The DEIP of Step~2 is run on
      $\lang^*_{J,A^{\text{in}}} = \{x\in\bin^n : C_0(x)|_J = A^{\text{in}}\}$,
  decided by a circuit that evaluates $C_0$ (depth at most $D$, size at
  most $S$) and compares its $J$-coordinates against the hardwired string $A^{\text{in}}$. Here $J$ is given by its succinct description $\lrag J$, which can be produced by a logspace-uniform circuit of size $\poly(|J|, \log \ell)$ and depth $\polylog(\ell)$. Equality check on $|J|$ queries can be implemented in size $O(|J|\cdot m)$ and depth $\log (|J|m)$. Therefore, since $|J| \le \ell \le S$, the comparison can be implemented in depth
  $\log (m|J|) + \polylog(\ell) \in \polylog(S)$ and size $\poly(|J|, m, \log \ell) \in \poly(S)$; hence this circuit has depth $D+\polylog(S)$ and size $\poly(S)$, on input $x$ of length $n$.
    \item \textbf{Efficiency.} The prover sends $A=Q|_J$, consisting of $|J| = O(\sigma/\eps)$
  query--answer pairs. The verifier makes $|J| = O(\sigma/\eps)$ oracle
  queries to check $y_i = O(x_i)$ for $i\in J$ and evaluates the IPP
  decision predicate on $A$, in time $O((\sigma/\eps) m \polylog(\ell))$.

  The parties then run the DEIP of \Cref{deIP} with soundness parameter $\sigma+2$ on the depth-$(D+\polylog(S))$,
  size-$\poly(S)$ circuit above: this contributes communication
  $O(\sigma\cdot D\cdot \polylog(S))$, round complexity $O(D\cdot\polylog(S))$, and verifier
  running time $O(\sigma\cdot(n + D\polylog(S)))$.
    \end{itemize}
  \end{itemize}

  Summing the two steps, and using $\eps\ell \ge 1$ to absorb the additive $O(\sigma\cdot D\cdot\polylog(S))$ costs of the DEIP into the $O(\sigma\cdot\eps\cdot m\cdot\ell\cdot D\cdot\polylog(S))$ costs of the IPP, we obtain:
  \begin{itemize}
  \item query complexity $q = O(\sigma/\eps)$;
  \item communication complexity
    $cc = O(\sigma\cdot\eps m \ell D\cdot\polylog(S)) + O((\sigma/\eps) \cdot m \polylog(\ell)) + O(\sigma\cdot D \cdot \polylog(S))
        = O(\sigma \cdot m \cdot \eps\cdot\ell \cdot D\cdot\polylog(S) + (\sigma / \eps) \cdot m \cdot \polylog(\ell))$;
  \item round complexity $O(D\cdot\polylog(S))$;
  \item verifier running time
    $O(\sigma\cdot(\eps m \ell D\cdot \polylog(S) + (1/\eps) m \polylog(S) + n\log S)) + O((\sigma/\eps) m \polylog(\ell)) + O(\sigma\cdot(n + D\polylog(S)))
     = O(\eps\ell \cdot D +1/\eps)\cdot \sigma \cdot m \cdot\polylog(S) + O(\sigma n\log S)$;
  \item prover running time $\poly(S, \sigma)$, as the prover evaluates $C^O(x)$ in time $\poly(S)$ and runs the IPP and DEIP provers, each in time $\poly(S, \sigma)$.
  \end{itemize}

\end{proof}

\subsection{A Proof System for Circuits Making Adaptive Oracle Queries}\label{sec:adaptive}

\begin{restatable}{theorem}{proofadaptive}\label{thm:proof-adaptive}


  Let $\eps = \eps(n)\in (0, 1], \sigma \in \NN$, and let
  $\mathcal{L}$ be a language decidable by $\eps$-robust logspace-uniform
  oracle Boolean circuits $C$ with $\ell$ $d$-adaptive oracle gates, depth $D$, and size $S$, where all gates except possibly the oracle
  gates have fan-in $2$. Assume that query lengths are bounded by $m = \poly(n)$, and that
  $\eps \ge 1/\ell$, as otherwise $\eps$-robustness holds vacuously
  (see \cref{defn:robustness}). Then there is a doubly-efficient interactive
  proof system for $\mathcal{L}$ with the following parameters:\footnote{We use $\widetilde{O}(\cdot)$ to hide $\polylog(d, S)$ factors.} 


    
  \begin{itemize}
    \item soundness error $2^{-\sigma}$;
    
    \item query complexity
    $\widetilde{O}\!\left(\sigma d/\eps\right)$;

    \item communication complexity
    $\widetilde{O}\!\left(
        \sigma d D m \eps \ell
        + \sigma d m / \eps
    \right)$;

    \item round complexity
    $\widetilde{O}(dD)$;

    \item verifier running time
    $\widetilde{O}\left(
        \sigma d\left(
            \eps m \ell D
            +
            m/\eps
            +
            n
        \right)
    \right)$; and

    \item prover running time
    $\poly(S,\sigma)$.
\end{itemize}

    
  
\end{restatable}

\begin{figure}[htbp]
\fbox{
\begin{minipage}{\linewidth}
  \RaggedRight \vspace{.25cm}
  \underline{\textbf{DEIP for robust, adaptive setting}}\\
  \vspace{.25cm} \textbf{Input:} The prover $\mathcal{P}$ and verifier
  $\mathcal{V}$ receive $x\in\{0,1\}^n$, a succinct
  description of an $\eps$-robust logspace-uniform oracle circuit $C$ with $\ell$ $d$-adaptive oracle gates, depth $D$, and size $S$, and oracle access to $O:\{0,1\}^*\to\{0,1\}$. Note that $\eps$ is a fixed parameter, while $\ell$, $D$, and $S$ are parameters of the input.\\
  \vspace{.2cm} \textbf{Additional parameters:} We let $\sigma$ denote the soundness parameter. For each level $i\in[d]$, let
  $\ell_i$ be the number of oracle gates in level $i$, $\eps_i$ be a proximity parameter, and $h_i$ be the checksum length. For all $i$, set $\eps_i = \eps/d$ and $h_i\le O(\varepsilon m\ell\log \ell)$ (its exact value is determined in the analysis). We use $\sigma_\text{IPP}$ and $\sigma_\text{DEIP}$ for the
  soundness parameters of the IPP and the DEIP, and set
  $\sigma_\text{IPP}, \sigma_\text{DEIP} = \sigmaIPP$. \\
  \vspace{.2cm} \textbf{Notations:} For each $i\in [d]$, enumerate the oracle gates in level $i$ by $(i, 1), \ldots, (i, \ell_i)$. For $i\in[d]$, $j\in[\ell_i]$, let
  $x_{i,j}$ denote the query made at oracle gate $(i,j)$ and
  $y_{i,j}=O(x_{i,j})$ the corresponding answer. Let
  $Q_i \coloneq ((x_{i,1},y_{i,1}),\ldots,(x_{i,\ell_i},y_{i,\ell_i}))$
  denote the string of oracle query-answer pairs for oracle gates in
  level $i$. For a query-answer string $S$, let $S^{\text{in}}$ denote
  the string of queries in $S$ and $S^{\text{out}}$ denote the string of
  answers. 
  Decompose $C^O$ into circuits $C_0,\ldots,C_d$ without
  oracle gates as follows: $C_0$ maps $x\mapsto Q_{1}^{\text{in}}$; for $i\in[d-1]$,
  $C_i$ maps $(x,Q_1, \ldots, Q_i)\mapsto Q_{i+1}^{\text{in}}$; and $C_d$ maps
  $(x,Q_1, \ldots, Q_d)\mapsto C^O(x)$.\\
  \vspace{.2cm} \textbf{The protocol:}
\begin{enumerate}[itemsep=.05cm]

\item The prover claims $C^O(x)=1$.
For each $i\in [d]$, the prover
  computes the checksum of $Q_i$ to $h_i$ bits and sends the
  tuple of checksums $(H_i = \cksum(Q_i))_{i\in[d]}$
  to the verifier.

\item The prover and the verifier run the
  message-exchange portion of the IPP from \cref{rothblumx2ipp} with proximity parameter
  $\eps_1+\cdots +\eps_d$ and soundness parameter $\sigma_{\text{IPP}}$ on explicit input $x$ and implicit input $(Q_1, \ldots, Q_d)$ to prove
  that $C_d(x,Q_1, \ldots, Q_d)=1$ and $Q_1, \ldots, Q_d$
  compute correctly to the checksums in step 1, where
  $(Q_1, \ldots, Q_d)$ is the implicit part of the input.
  At the end of the IPP, the verifier needs to query sets of indices
  $J_{d, 1}\subseteq [\ell_1], \ldots, J_{d, d}\subseteq [\ell_d]$ in $Q_1, \ldots, Q_d$, respectively. 

\item The verifier asks the prover to send the
  query-answer pairs in $Q_i$ at indices $J_{d, i}$ for each
  $i\in [d]$. For each $i$, the prover sends
 a string $A_{d, i}\coloneq Q_i|_{J_{d, i}}$, and the verifier checks that
  all the oracle calls in every $A_{d, i}$ are correct.

\item For $j=d-1,d-2,\ldots,1$, the prover and the verifier run an IPP with proximity parameter
  $\eps_1+\cdots +\eps_j$ and soundness parameter $\sigma_{\text{IPP}}$ on input
  $(x, J_{j+1}, A_{j+1}^{\text{in}}, Q_1, \ldots, Q_j)$ to show that
  $C_j(x,Q_1, \ldots, Q_j)|_{J_{j+1}} = A_{j+1}^{\text{in}}$ and
  $Q_1, \ldots, Q_j$ compute correctly to the checksums in step 1,
  where $J_{j+1} = \bigcup_{k = j+1}^d J_{k, j+1}$ consists of all the
  indices that the verifier has queried in
  $Q_{j+1}$ in previous IPPs, $A_{j+1}$ is the sorted concatenation of
  $A_{d, j+1}, A_{d-1, j+1}, \ldots, A_{j+1, j+1}$, i.e., all the strings of
  answers that the prover provides to queries to $Q_{j+1}$ in previous
  IPPs, and $(Q_1, \ldots, Q_j)$ is the implicit part of the input. At
  the end, the verifier outputs sets
  $J_{j, 1}\subseteq[\ell_1], \ldots, J_{j, j}\subseteq[\ell_j]$ of indices that it
  wishes to query in $Q_1, \ldots, Q_j$ respectively. For each $i\in [j]$, the prover sends
  $A_{j, i} := Q_i|_{J_{j, i}}$, and the verifier checks that the oracle calls in
  $A_{j, i}$ are correct.

\item Finally, letting $J_1$ denote $\bigcup_{k = 1}^d J_{k, 1}$ and
  $A_1$ denote the sorted concatenation of
  $A_{d, 1}, \ldots, A_{1, 1}$, the prover and the verifier engage
  in a DEIP (by \Cref{deIP}) with soundness parameter $\sigma_{\text{DEIP}}$ to prove that $C_0(x)|_{J_1} = A_{1}^{\text{in}}$.

  \item The verifier accepts if the IPP and DEIP verifiers accept and all of its checks pass, and rejects
    otherwise.
  
\end{enumerate}  

\end{minipage}
}
\caption{DEIP for robust circuits making adaptive
  queries}\label{fig:fig-bd-adaptive}
\end{figure}  

\begin{proof}

  The protocol is given in Figure~\ref{fig:fig-bd-adaptive}. Here, we analyze
  its completeness, soundness, and efficiency.

  \paragraph{Completeness.} If $x$ is in the language, then the honest
  prover answers the queries to the inputs in the IPPs consistent with
  the true query-answer string $Q$. Completeness then follows from
  completeness of the IPPs and the DEIP.

  \paragraph{Soundness.} Fix any cheating prover $\MProver$ and $x\notin\mathcal{L}$, and let $E$ be the event that $\mathcal{V}$ accepts. We show that $\text{Pr}[E]\le 2^{-\sigma}$. 


  We first define some helpful notations. The protocol contains $d$
  IPPs, one for each level of oracle gates in the circuit. We refer to
  the IPP on implicit input $(Q_1, \ldots, Q_i)$ as IPP $i$ and denote
  the IPP verifier for IPP $i$ by $\mathcal{V}_i$. The IPPs in our
  protocol differ from standard IPPs in that the verifier does not
  have query access to the input: rather, the verifier asks the
  cheating prover $\MProver$ for positions in the
  input, and $\MProver$'s answers may not be consistent
  with the input. We use $\widetilde{A}_i$ for
  $\MProver$'s string of answers $A_i$ to queries to
  $Q_i$. Recall we use $\sigma_\text{IPP}$ and $\sigma_\text{DEIP}$ for the
  soundness parameters of the IPP and the DEIP, and set
  $\sigma_\text{IPP}, \sigma_\text{DEIP} = \sigma + \log 2d$. 

We recursively define strings $R_1,\ldots,R_d$ as follows. Set
$R_1=Q_1$. For $i=2,\ldots,d$, if there exists a tuple
$(R_1',\ldots,R_{i-1}')$ such that
(i) $(R_1',\ldots,R_{i-1}')$ is
$\sum_{j=1}^{i-1}\eps_j$-close to $(R_1,\ldots,R_{i-1})$, and
(ii) each $R_j'$ computes correctly to the checksums that the prover sent in
Step 1,
then define
\[
  R_i := \big(C_{i-1}(x,R_1',\ldots,R_{i-1}'),\,
              O(C_{i-1}(x,R_1',\ldots,R_{i-1}'))\big),
\]
where $O(C_{i-1}(x,R_1',\ldots,R_{i-1}'))$ is obtained by applying $O$ to every query in $C_{i-1}(x,R_1',\ldots,R_{i-1}')$. (If more than one such tuple exists, we define $R_i$ using the lexicographically first one; if no such tuple exists, we set $R_i := \bot$, where by convention no query-answer string is close to, or consistent with, $\bot$.) We will
  show that $\MProver$'s answers
  $\widetilde{A}_1, \ldots, \widetilde{A}_d$ are with good probability
  consistent with $R_1, \ldots, R_d$, which means that even though our
  IPPs differ from standard IPPs in that the prover holds the implicit
  input, from $\mathcal{V}_i$'s point of view, IPP $i$ looks the same
  as standard IPP on implicit input $(R_1, \ldots, R_i)$.
  
  Towards this, we define some helpful events:

  \begin{itemize}

  \item For $i\in [d]$, let $F_i$ be the event that $\widetilde{A}_i$
    is consistent with $R_i$.

  \item For $i\in [d]$, let $U_i$ be the event that there is at most
    one query-answer string $R_i'$ such that $R_i'$ is
    $\sum_{j = 1}^d \eps_j$-Hamming-close to $R_i$ and $R_i'$ computes
    correctly to the checksums that the prover sent in step 1.
    
  \end{itemize}

  Consider the following claims:

  \begin{claim}\label{clm:pa1}
  We have $\text{Pr}[E\cap\overline{F_1}]\le 2^{-\sigma_\text{DEIP}}$.  
  \end{claim}

  \begin{claim}\label{clm:pa2}
  For $i\in\{2, \ldots, d\}$, we have 
  \[
\Pr\!\left[
    E \cap
    \left(
        \overline{F_i}
        \cap
        \bigcap_{j=1}^{i-1}
        (F_j \cap U_j)
    \right)
\right]\le 2^{-\sigma_\text{IPP}}.
    \]
  \end{claim}

  \begin{claim}\label{clm:pa3}
  We have 
  \[
  \Pr\!\left[
    E \cap \bigcap_{i=1}^{d} (F_i \cap U_i)
\right]\le 2^{-\sigma_\text{IPP}}.
  \]
  \end{claim}

  \begin{claim}\label{clm:pa4}
  For all $i\in [d]$, with $h_i\le O(\varepsilon m \ell\log \ell)$, we have $\text{Pr}[E\cap \overline{U_i}] = 0$.  
  \end{claim}

We first show that, assuming these claims, we have $\Pr[E]\le 2^{-\sigma}$. We have 
\begin{align*}
\Pr[E]
&=
\Pr\!\left[
    E \cap \bigcap_{i=1}^{d} (F_i \cap U_i)
\right]
+
\Pr\!\left[
    E \cap
    \left(\overline{
        \bigcap_{i=1}^{d} (F_i \cap U_i)
    }\right)
\right] \\
&=
\Pr\!\left[
    E \cap \bigcap_{i=1}^{d} (F_i\cap U_i)
\right]
+
\Pr\!\left[
    E \cap
    \bigcup_{i=1}^{d}
    \left(
        \overline{F_i}
        \cup
        \overline{U_i}
    \right)
\right],
\end{align*}
where the second step applies De Morgan's. By \cref{clm:pa3}, the first summand is at most $2^{-\sigma_\text{IPP}}$. We can bound the second summand as follows:
\[
\Pr\!\left[
    E \cap
    \bigcup_{i=1}^{d}
    \left(
        \overline{F_i}
        \cup
        \overline{U_i}
    \right)
\right]
\le
\Pr[E \cap \overline{F_1}]
+
\sum_{i=1}^{d}
\Pr[E \cap \overline{U_i}]
+
\sum_{j=2}^{d}
\Pr\!\left[
    E \cap
    \left(
        \overline{F_j}
        \cap
        \bigcap_{k=1}^{j-1}
        (F_k \cap U_k)
    \right)
\right].
\]
Applying \cref{clm:pa1}, \cref{clm:pa2}, and \cref{clm:pa4}, we get that this is at most $2^{-\sigma_\text{DEIP}}+(d-1)\cdot 2^{-\sigma_\text{IPP}}+0$. Putting everything together, we get $\Pr[E]\le 2^{-\sigma_\text{DEIP}}+d\cdot 2^{-\sigma_\text{IPP}}\le 2^{-\sigma}$, as desired.

To finish, we prove the claims.

\begin{proof}[Proof of \cref{clm:pa1}]
Since $E$ holds (i.e., the verifier accepts), the DEIP verifier accepts. Then since $R_1=Q_1$, the DEIP verifies that $C_0(x)|_{J_1} = \widetilde{A}_1^{\text{in}}$, and the verifier
also checks the oracle calls in $\widetilde{A}_1$, we have $\text{Pr}[E\cap \overline{F_1}]\le 2^{-\sigma_\text{DEIP}}$.
\end{proof}

\begin{proof}[Proof of \cref{clm:pa2}]

Let $i\in\{2, \ldots, d\}$. Since $F_1\cap\cdots\cap F_{i-1}$ holds, $\mathcal{V}_{i-1}$ sees answers exactly as if it had query access to the implicit input
$(R_1,\ldots,R_{i-1})$. Since $E$ holds, $\mathcal V_{i-1}$
accepts. Thus, by IPP soundness, except with probability
$2^{-\sigma_\text{IPP}}$, there exists a tuple
$(R_1',\ldots,R_{i-1}')$ that is
$\sum_{j=1}^{i-1}\eps_j$-close to $(R_1,\ldots,R_{i-1})$, computes
correctly to the checksums in step 1, and satisfies
the relation checked by IPP $i-1$. Since $U_1\cap\cdots\cap U_{i-1}$ holds, each such $R_j'$ is unique; in particular, $(R_1', \ldots, R_{i-1}')$ is exactly the tuple used to define $R_i$, which means that
  $R_i = (C_{i-1}(x, R_1', \ldots, R_{i-1}'), O(C_{i-1}(x, R_1', \ldots, R_{i-1}')))$. Then since $\widetilde{A}_i^{\text{in}} =
  C_{i-1}(x, R_1', \ldots, R_{i-1}')|_{J_i}$ and the verifier checks that all the oracle calls in $\widetilde{A_i}$ are correct, 
  $\widetilde{A_i}$ is consistent with $R_i$, so we have 
  \[
\Pr\!\left[
    E \cap
    \left(
        \overline{F_i}
        \cap
        \bigcap_{j=1}^{i-1}
        (F_j \cap U_j)
    \right)
\right]\le 2^{-\sigma_\text{IPP}}.
    \]

\end{proof}

\begin{proof}[Proof of \cref{clm:pa3}]

Since $F_1\cap\cdots\cap F_d$ holds, the final IPP
verifier $\mathcal V_d$ has exactly the view of a standard IPP verifier with
query access to implicit input $(R_1,\ldots,R_d)$. If $E$, then $\mathcal V_d$ accepts. Therefore, except with probability
$2^{-\sigma_\text{IPP}}$, there exists a tuple
$(R_1',\ldots,R_d')$ that is $\sum_{j=1}^d\eps_j$-close to
$(R_1,\ldots,R_d)$, computes correctly to the checksums in step 1, and satisfies
\[
  C_d(x,R_1',\ldots,R_d')=1.
\]
Since $U_1\cap\cdots\cap U_d$ holds, this tuple is unique; in particular, for each $i$, the prefix $(R_1',\ldots,R_{i-1}')$ is exactly the tuple used to define $R_i$.

  Then
  $(R_1', \ldots, R_d')$ is
$\sum_{i=1}^d\eps_i$-close to the true query-answer string
$(Q_1,\ldots,Q_d)$ in the sense of \cref{defn:closeness}: starting from $R_1 = Q_1$, we can modify positions to obtain $R_1'$, then recompute the queries and answers for the next level from $(x, R_1')$ and modify positions to obtain $R_2'$, and continue in this way through level $d$. In total, we modify at most an
  $\sum_{i = 1}^d \eps_i$-fraction of positions. Since $C$ is $\eps$-robust and $C_d(x, R_1', \ldots, R_d') = 1$, assuming $\sum_{i = 1}^d \eps_i\le\eps$, we have that
  $C^O(x) = 1$, which is a contradiction, implying 
  \[
  \Pr\!\left[
    E \cap \bigcap_{i=1}^{d} (F_i\cap U_i)
\right]\le 2^{-\sigma_\text{IPP}}.
  \]
\end{proof}

\begin{proof}[Proof of \cref{clm:pa4}]


If \(\overline{U_i}\) occurs, then there exist two distinct strings
\(R'_i,R''_i\) $\sum_{j=1}^d \varepsilon_j$-Hamming-close to $R_i$ that compute correctly to the checksums sent by the prover in step 1, i.e.,
$\cksum(R'_i)=\cksum(R''_i)$.

Choosing $h_i\le O(\varepsilon \ell \cdot m \log \ell)$ (e.g., we can set $h_i = 4 \varepsilon \ell \cdot m \log \ell$), which guarantees $\varepsilon$-unique decoding checksums by \Cref{lem:unique-decoding}, we get
that $\Pr[\overline{U_i}] = 0$, which implies $\Pr[E\cap\overline{U_i}] = 0$.

\end{proof}

\paragraph{Efficiency.} Finally, we analyze the efficiency of the protocol. For $i\in [d]$, we take $\eps_i = \eps / d$ and 
$h_i = O(\varepsilon m \ell\log \ell)$.

We break the protocol into the following parts:

\begin{itemize}
\item At the start of the protocol, the prover sends a checksum of each level of oracle gates to the verifier. This involves $O(1)$ rounds and communication 
\[
    \sum_{i=1}^d h_i
    \le
    O\left(d \cdot \varepsilon\ell\log \ell \cdot m\right).
\]

\item The protocol contains $d$ IPPs. We refer to the first IPP as IPP-$d$ and the second IPP as IPP-$(d-1)$ and so on.
 For $i\in [d-1]$, IPP-$i$ is run on language $\mathcal{L}_{J_{i+1}, A_{i+1}^{\text{in}}}\coloneq\{x\in\{0, 1\}^n\mid C_i(x, Q_1, \ldots, Q_i)|_{J_{i+1}} = A_{i+1}^{\text{in}} \text{ and } \cksum(Q_j) = H_j \text{ for all } j \in [i]\}$, which is decided by a circuit that runs $C_i$ on input $(x, Q_1, \ldots, Q_i)$, does an equality check to verify that its output at the indices in $J_{i+1}$ matches $A_{i+1}^{\text{in}}$, and evaluates the checksums of $Q_1, \ldots, Q_i$ and compares them to the checksums $H_1, \ldots, H_i$ sent in step 1. Since each checksum is an $\FF$-linear map of its input, it can be evaluated with $\poly(S)$ gates in $\polylog(S)$ depth, so such a circuit can be implemented with $\poly(S)$ gates and $O(D+\polylog(S))$ depth. IPP-$d$ is run on the language $\{x\in\{0, 1\}^n\mid C_d(x, Q_1, \ldots, Q_d) = 1 \text{ and } \cksum(Q_j) = H_j \text{ for all } j \in [d]\}$, and this can be decided by a circuit that runs $C_d$ on input $(x, Q_1, \ldots, Q_d)$, checks that the output is $1$, and checks the checksums in the same way. Such a circuit can also be implemented with $O(\poly(S))$ gates and $O(D+\polylog(S))$ depth. Additionally, for $i\in [d]$, IPP-$i$ has explicit input $(x, J_{i+1}, A_{i+1}^{\text{in}}, H_1, \ldots, H_i)$ (just $(x, H_1, \ldots, H_d)$ for $i = d$), whose length beyond $n$ only contributes terms dominated by the other costs accounted below, implicit input length $\sum_{j = 1}^i \ell_j$, proximity parameter $\sum_{j = 1}^i \eps_j = i\eps/d$, and soundness parameter $\sigma_\text{IPP} = \sigma + \log 2d$. Let
\[
    L_i \coloneq \sum_{j=1}^i \ell_j,
    \qquad
    \alpha_i \coloneq \sum_{j=1}^i \eps_j = \frac{i\eps}{d},
    \qquad
    \sigma_{\mathrm{IPP}} \coloneq \sigma+\log(2d).
\]
Applying \cref{rothblumx2ipp} and summing over all $d$ IPPs, we get that the IPPs contribute the following towards the efficiencies:

\begin{itemize}

\item query complexity 
\[
    \sum_{i=1}^d O\left(\frac{\sigma_{\mathrm{IPP}}}{\alpha_i}\right)
    =
    O\left(
        \frac{\sigma_{\mathrm{IPP}}d\log d}{\eps}
    \right).
\]

\item communication complexity 
\begin{align*}
    & \sum_{i=1}^d
    O\left(
        \sigma_{\mathrm{IPP}}\alpha_i m L_i
        (D+\polylog(S))\polylog(S)
        + \sigma_{\mathrm{IPP}} / \alpha_i \cdot m
    \right) \\
    =& 
    O\left(
        \sigma_{\mathrm{IPP}}\eps m \ell d
        (D+\polylog(S))\polylog(S) + \sigma_{\mathrm{IPP}} d \log d \cdot m / \eps
    \right),
\end{align*}
where we use that $L_i\le \ell$ for every $i$ and $\sum_{i=1}^d 1/\alpha_i = O((d \log d)/\eps)$,

\item round complexity 
\[
    O\left(
        d(D+\polylog(S))\polylog(S)
    \right),
\]

\item verifier running time 
\[
    \sum_{i=1}^d
    O\left(
        \sigma_{\mathrm{IPP}}\left(
            \alpha_i m L_i(D+\polylog(S))\polylog(S)
            +
            \frac{m}{\alpha_i}\polylog(S)
            +
            n\log S
        \right)
    \right),
\]
which is
\[
    O\left(
        \sigma_{\mathrm{IPP}}\left(
            \frac{\eps(D+\polylog(S))\polylog(S)}{d}
            \sum_{i=1}^d iL_i \cdot m
            +
            \frac{d\log d}{\eps} m\polylog(S)
            +
            dn\log S
        \right)
    \right),
\]
which, simplifying, is
\[
    O\left(
        \sigma_{\mathrm{IPP}}\left(
            \eps m \ell d(D+\polylog(S))\polylog(S)
            +
            \frac{d\log d}{\eps} m \polylog(S)
            +
            dn\log S
        \right)
    \right),
\]
and

\item prover running time $\poly(S,\sigma_{\mathrm{IPP}},d)$, which is
$\poly(S,\sigma)$.
\end{itemize}
At the end, the verifier needs to send the prover the queries that he wants to make to the implicit input. This cost is swallowed by the communication complexity of the IPP.  

\item The protocol contains one DEIP, run on the language $\mathcal{L}_{J_1, A_1^\text{in}}\coloneq\{x\in\{0, 1\}^n\mid C_0(x)|_{J_1} = A_1^{\text{in}}\}$, which can be decided by a circuit that runs $C_0$ on $x$ and performs an equality check to verify that its output at the indices in $J_1$ matches $A_1^{\text{in}}$. Such a circuit can be implemented with $O(\poly(S))$ gates and $O(D+\polylog(S))$ depth. Applying \cref{deIP} with soundness parameter $\sigma_{\text{DEIP}} = \sigma + \log 2d$, we get that the DEIP contributes the following:
\begin{itemize}
\item communication complexity
\[
    O\!\left(\sigma_{\mathrm{DEIP}}(D+\polylog(S))\polylog(S)\right)
\]

\item round complexity 
\[
    O\!\left((D+\polylog(S))\polylog(S)\right)
\]

\item verifier running time
\[
    O\!\left(\sigma_{\mathrm{DEIP}}\bigl(n+(D+\polylog(S))\polylog(S)\bigr)\right)
\]

\item prover running time
\[
    \poly(S,\sigma)
\]

\end{itemize}

\end{itemize}

Putting everything together, using $\widetilde{O}(\cdot)$ to hide $\polylog(d)$ and $\polylog(S)$ factors, and using $\eps\ell \ge 1$ to absorb the additive costs of the DEIP into the costs of the IPPs, we get the following complexities:

\begin{itemize}
    \item query complexity
    \[
        O\left(
            \frac{(\sigma+\log d)d\log d}{\eps}
        \right) 
        =
        \widetilde{O}\left(
            \frac{\sigma d}{\eps}
        \right)
    \]
     \item communication complexity
    \begin{align*}
    &O\left(
        d(\eps m \ell\log\ell)
        +
        (\sigma+\log d)\eps m \ell d
        (D+\polylog(S))\polylog(S)
        + (\sigma + \log d) d \log d \cdot m / \eps
    \right) \\
    &\qquad=
    \widetilde{O}\left(
        \sigma d D m \eps \ell
        + \sigma d m / \eps
    \right)
\end{align*}
     \item round complexity
    \[
        O\left(
            d(D+\polylog(S))\polylog(S)
        \right) 
        =
        \widetilde{O}(dD)
    \]

    \item verifier running time
    \begin{align*}
    &O\left(
        (\sigma+\log d)
        \left(
            \eps m \ell d(D+\polylog(S))\polylog(S)
            +
            \frac{d\log d}{\eps} m \polylog(S)
            +
            dn\log S
        \right)
    \right) \\
    &\qquad=
    \widetilde{O}\left(
        \sigma d\left(
            \eps m \ell D
            +
            \frac{m}{\eps}
            +
            n
        \right)
    \right)
\end{align*}

    \item prover running time
    \[
        \poly(S,\sigma)
    \]
\end{itemize}

\end{proof}  

\subsection{An Argument System for Circuits Making Adaptive Oracle Queries}\label{sec:arg} 

\begin{restatable}{theorem}{argadaptive}\label{thm:arg-adaptive}

  Let $\eps = \eps(n)\in (0, 1]$ be a proximity parameter, and let
  $\mathcal{L}$ be a language decidable by $\eps$-robust logspace-uniform oracle Boolean circuits with $\ell = \ell(n)$ oracle gates and of polynomial size $S = S(n) \leq 2^\secp$. Assume all query lengths are bounded by $m = \poly(n)$. Then there is a
  doubly-efficient interactive argument system for $\mathcal{L}$ with the
  following parameters:

  \begin{itemize}

  \item soundness error $1/2$ for any PPT adversary;
  \item query complexity $q = O(1/\eps)$;
  \item communication complexity $cc = (\poly(\secp) + m )\cdot 1 / \eps$;
  \item round complexity $O(1)$;
  \item verifier running time $\poly(\secp) \cdot \left( 1 / \eps + n \right) + m / \eps$; and
  \item prover running time $\poly(S, \secp) \cdot 1 / \eps$.
    
  \end{itemize}  
  
\end{restatable}  

\begin{figure}[htbp]
\fbox{
\begin{minipage}{\linewidth}
  \RaggedRight
  \vspace{.3cm}
  \underline{\textbf{Doubly-efficient interactive argument for robust, adaptive setting}}\\
  \vspace{.3cm}
  \textbf{Input:} The prover $\mathcal{P}$ and
    verifier $\mathcal{V}$ receive $x\in\{0,1\}^n$, a succinct
    description of a logspace-uniform oracle circuit $C$ containing $\ell$ oracle gates, and oracle access to $O:\{0,1\}^*\to\{0,1\}$.\\
  \vspace{.2cm}
  \textbf{Additional parameters:} We take $q = 2 / \eps$. \\
  \vspace{.2cm}
  \textbf{Notations and tools:} Enumerate the oracle gates of $C^O$ in topological order by
  $1, \ldots, \ell$. On input $x$, for each $i\in [\ell]$, let $x_i$
  denote the input to the $i$th oracle gate and $y_i = O(x_i)$ the corresponding oracle answer. Let $\mathcal{M}$ denote the set of oracle input-output pairs. We will use the following tools:
  \begin{itemize}
    \item A succinct argument of knowledge $\text{KIL} = (\kilsetup, \kilprove, \kilver, \kilext)$ for $\mathsf{NP}$ relations (see \Cref{sec:prelims:AoK}).
    \item A hash tree family $\HT = (\htgen, \hthash, \htprove, \htver)$ (see \Cref{def:hash_tree}). 
  \end{itemize}
\textbf{The protocol:}
  \begin{enumerate}[itemsep=.1cm]
    \item The verifier generates public parameters $\pp_{\kil} \gets \kilsetup(1^\sec)$, and $\hk \gets \htgen(\secparam, \ell)$, where $\hk$ are hash keys for database over $\mathcal{M}$.
    \item The prover simulates the entire execution $C^O(x) = 1$, computing all intermediate oracle queries $x_1, \ldots, x_\ell$ and answers $y_1 = O(x_1), \ldots, y_\ell = O(x_\ell)$. The prover then computes the hash tree root $\rt \gets \hthash(\hk, \Set{ (x_1, y_1), \ldots, (x_\ell, y_\ell)})$ and sends $\rt$ to the verifier.
    \item The prover and verifier engage in Kilian's protocol to prove that there exists a set $\Set{ (x_1^*, y_1^*), \ldots, (x_\ell^*, y_\ell^*)}$ that is (i) consistent with $\rt$ and (ii) on input $x$, taking the oracle answers to be those in $\Set{y_i^*}_{i\in [\ell]}$, $C$ makes the queries $\Set{x_i^*}_{i\in [\ell]}$ and accepts. In particular, they prove the following relation
    \[ \Rel\coloneq\{(\hk, (C, x, \rt), \Set{ (x_1^*, y_1^*), \ldots, (x_\ell^*, y_\ell^*)} ) \mid \rt = \hthash(\hk, \Set{ (x_1^*, y_1^*), \ldots, (x_\ell^*, y_\ell^*)}) \wedge C^{x_i^* \to y_i^*}(x) = 1\}, \]
    where $C^{x_i^* \to y_i^*}(x) = 1$ means all queries made by $C(x)$ are $\Set{x_i^*}_{i \in [\ell]}$ and given that all query answers are $\Set{y_i^*}_{i \in [\ell]}$, $C$ accepts. 
    \item The verifier randomly samples $q$ indices $J \subseteq [\ell]$ and asks the prover to open the hash tree on these indices. The prover sends $\Set{(x_j^*, y_j^*, \rho_j)}_{j \in J}$, where $\rho_j$ is the hash tree opening proof for $(x_j^*, y_j^*)$. The verifier checks that for all $j \in J$, $\htver(\hk, \rt, j, (x_j^*, y_j^*), \rho_j) = 1$ and $y_j^* = O(x_j^*)$.
    \item The verifier accepts if all checks pass, and rejects otherwise.
    \end{enumerate}

\end{minipage}
}
\caption{Doubly-efficient interactive argument for robust circuits making adaptive queries}\label{fig:fig-adaptive}
\end{figure}  

\begin{proof}

  The protocol appears in Figure~\ref{fig:fig-adaptive}. Here we
  analyze its efficiency, completeness, and soundness.

  \paragraph{Efficiency.} The communication is composed of:
  \begin{itemize}
    \item Public parameters $\pp_{\kil}, \hk$ of size $\poly(\secp)$.
    \item Hash root $\rt$ of size $\poly(\secp)$.
    \item The Kilian's protocol interaction of size $\poly(\secp)$.
    \item $q$ oracle input-output pairs and their openings $\Set{(x_j^*, y_j^*, \rho_j)}_{j \in J}$, where $q = O(1 / \eps)$. Their total size is $(\poly(\secp) + m) \cdot 1 / \eps$.
  \end{itemize}
  Thus the total communication is $(\poly(\secp) + m) \cdot 1 / \eps$.

  During the above protocol, the verifier runs in time $\poly(\secp) \cdot (1 / \eps + n) + m / \eps$ and the prover runs in time $\poly(S, \secp) \cdot 1 / \eps$.

  \paragraph{Completeness.} The completeness follows from the completeness of Kilian's protocol and the hash tree.

  \paragraph{Soundness.} Let $\MProver$ be any poly-time malicious prover, and we rewind $\MProver$ (just after sending $\rt$) to obtain a prover for the succinct argument of knowledge $\MProver_{1}$. Specifically, we run $\kilext^{\MProver_1}$ to extract $\Set{ (x_1^*, y_1^*), \ldots, (x_\ell^*, y_\ell^*)}$ that is consistent with $\rt$ and the computation. By knowledge soundness, we have 

  \[
  \probcond{
    \langle \MProver, \Verifier \rangle = 1 \wedge \\
    \big( \rt \neq \hthash(\hk, \Set{ (x_1^*, y_1^*), \ldots, (x_\ell^*, y_\ell^*)}) \\
     \vee C^{x_i^* \to y_i^*}(x) = 0 \big)
  }{
    \pp_\kil \gets \kilsetup(1^\sec), \hk \gets \htgen(\secparam, \ell) \\
    \Set{ (x_1^*, y_1^*), \ldots, (x_\ell^*, y_\ell^*)} \gets \kilext^{\MProver_1}
  } \leq \negl(\secp).
  \]

  Now we claim that for at most an $\eps$ fraction of $\Set{ (x_1^*, y_1^*), \ldots, (x_\ell^*, y_\ell^*)}$, it holds that $y_i^* \neq O(x_i^*)$. 
  
  \begin{claim}
    For any poly-time malicious prover $\MProver$, it holds that 
    \[
    \probcond{
      \langle \MProver, \Verifier \rangle = 1 \\
      \wedge \rt = \hthash(\hk, \Set{ (x_1^*, y_1^*), \ldots, (x_\ell^*, y_\ell^*)}) \\
      \wedge \big| \Set{ i \in [\ell] : y_i^* \neq O(x_i^*) } \big| > \eps \cdot \ell
    }{
      \pp_\kil \gets \kilsetup(1^\sec), \hk \gets \htgen(\secparam, \ell) \\
      \Set{ (x_1^*, y_1^*), \ldots, (x_\ell^*, y_\ell^*)} \gets \kilext^{\MProver_1}
    } \leq \frac 1 3
    \]
  \end{claim}

  \begin{proof}
    We consider two events: 
    \begin{itemize}
      \item $E_0$: the sampled $q$ indices $J$ contains at least one index $j$ such that $y_j^* \neq O(x_j^*)$.
      \item $E_1$: the sampled $q$ indices $J$ contains no index $j$ such that $y_j^* \neq O(x_j^*)$.
    \end{itemize}

    If event $E_0$ happens and the verifier accepts, then $\MProver$ must break the binding property of the hash tree. That is,

    \[
    \probcond{
      \langle \MProver, \Verifier \rangle = 1 \\
      \wedge \rt = \hthash(\hk, \Set{ (x_1^*, y_1^*), \ldots, (x_\ell^*, y_\ell^*)}) \\
      \wedge \big| \Set{ i \in [\ell] : y_i^* \neq O(x_i^*) } \big| > \eps \cdot \ell \\
      \wedge E_0
    }{
      \pp_\kil \gets \kilsetup(1^\sec), \hk \gets \htgen(\secparam, \ell) \\
      \Set{ (x_1^*, y_1^*), \ldots, (x_\ell^*, y_\ell^*)} \gets \kilext^{\MProver_1}
    } \leq \negl(\secp).
    \]

    On the other hand, the probability for $E_1$ to happen is small. In particular, $\Pr[E_1] \leq (1 - \eps)^q$. Since we set $q = \frac{2}{\eps}$, we have $\Pr[E_1] \leq 1 / e^2$.

    Combining the two cases we have 
    \[
    \probcond{
      \langle \MProver, \Verifier \rangle = 1 \\
      \wedge \rt = \hthash(\hk, \Set{ (x_1^*, y_1^*), \ldots, (x_\ell^*, y_\ell^*)}) \\
      \wedge \big| \Set{ i \in [\ell] : y_i^* \neq O(x_i^*) } \big| > \eps \cdot \ell
    }{
      \pp_\kil \gets \kilsetup(1^\sec), \hk \gets \htgen(\secparam, \ell) \\
      \Set{ (x_1^*, y_1^*), \ldots, (x_\ell^*, y_\ell^*)} \gets \kilext^{\MProver_1}
    } \leq \negl(\secp) + \frac 1 {e^2} \leq \frac 1 3.
    \]
  \end{proof}

  Therefore, if $\MProver$ convinces the verifier, then it must be the case that $C^O(x)$ accepts with at most an $\eps$ fraction of incorrect oracle answers, that is, with probability at most $1/3 + \negl(\secp) \leq 1/2$, we have $C^{x_i^* \to y_i^*}(x) = 1$ and $\big| \Set{ i \in [\ell] : y_i^* \neq O(x_i^*) } \big| \leq \eps \cdot \ell$, which contradicts with the fact that $C$ is $\eps$-robust and $x \notin \lang$. This finishes our soundness analysis.
\end{proof}

\section{A Doubly-Efficient Interactive Argument for Circuits with a Low-Degree Oracle}\label{sec:low-deg}
Here we present our doubly-efficient argument system for circuits with access to a low-degree oracle.

\begin{definition}[Degree of Oracle]\label{def:oracle-degree}
    Let $O: \bin^* \to \bin$ be an oracle and $O_n$ be the restriction of $O$ to inputs of length at most $n$. We say that $O$ has degree $d(n)$ if for every $n \in \bbN$, there exists a degree-$d(n)$ polynomial $P_n$ over some finite field $\mathbb{F}$ of size $2^{O(n)}$, such that we can encode every $x \in \bin^{\leq n}$ into $\mathbb{F}$ and $O_n(x) = P_n(x)$ for every such $x$.
\end{definition}

\begin{remark}[Oracle Access at Field Points]\label{rem:field-access}
    Throughout this section, we adopt the convention that oracle access to a degree-$d(n)$ oracle $O$ means oracle access to $P_n$: the parties may query the oracle at any point $z \in \mathbb{F}$ and receive $P_n(z)$, which on encoded binary inputs coincides with $O$.
\end{remark}

\begin{remark}[Field Size]\label{rem:field-size}
    Note that $\log |\mathbb{F}| \geq m$ always holds: the encoding of $\bin^{\leq m}$ into $\mathbb{F}$ is injective, and if $\log |\mathbb{F}| < m$, then field elements would be shorter descriptions of the queries than the queries themselves, which is impossible by counting.
\end{remark}
\begin{theorem}\label{thm:low-degree}
Let $O: \bin^* \to \bin$ be an oracle with degree $d(n) = \poly(n)$, and let
  $\mathcal{L}$ be a language decidable by logspace-uniform
  oracle Boolean circuits with $S = \poly(n)$ gates. Take $\lambda$ as the security parameter and $m$ as the maximum query length to the oracle. We assume that the field $\mathbb{F}$ over which $O_m$ is represented satisfies $\log |\mathbb{F}| \geq \lambda$. Assuming the polynomial hardness of $\SIS$, there exists an interactive argument system for $\mathcal{L}$ with the following parameters:

  \begin{itemize}

  \item query complexity $q = 1$;
  \item communication complexity $cc = \poly(\secp) \cdot m$;
  \item round complexity $O(\log n)$;
  \item verifier running time $\poly(\secp) \cdot (n+m)$; and
  \item prover running time $\poly(\secp, S)$.
    
  \end{itemize} 
\end{theorem}

\begin{figure}[htbp]
\fbox{
\begin{minipage}{\linewidth}
  \RaggedRight
  \vspace{.3cm}
  \underline{\textbf{Doubly-efficient interactive argument for circuits with access to a low-degree oracle}}\\
  \vspace{.3cm}
  \textbf{Input:} The prover $\mathcal{P}$ and verifier
  $\mathcal{V}$ receive $x\in\{0,1\}^n$, a succinct description of an
  oracle circuit $C$ of size $S$, and oracle access to
  $O:\{0,1\}^*\to\{0,1\}$ with degree $d(n)=\poly(n)$.\\
  \vspace{.2cm}
  \textbf{Notations:} Let $m=m(\secp)$ be the maximum input length to
  the oracle. Let $\mathbb{F}=\mathbb{F}_{\secp}$ be the field over
  which $O_m$ is represented as a polynomial, and assume that
  $\log |\mathbb{F}| \geq \max \Set{m, \lambda}$.\\
  \vspace{.2cm}
  \textbf{Ingredients:} We make use of the following tools:
  \begin{itemize}
    \item A polynomial commitment scheme
    $\PC=(\Setup_{\mathrm{CM}},\Commit,\Open,\Setup_{\mathrm{IP}},
    \calP_{\PC},\calV_{\PC})$ with efficient opening and verification
    for degree-bounded polynomials over $\mathbb{F}$, as in
    \Cref{thm:prelim_poly_commitment}, instantiated with degree bound
    $S$ for committing to $G$ and degree bound $S\cdot d(m)$ for
    committing to $F$.
    \item A succinct argument of knowledge
    $\mathrm{KIL}=(\kilsetup,\kilprove,\kilver,\kilext)$ for
    $\NP$ relations, as in \Cref{sec:prelims:AoK}.
  \end{itemize}
  \textbf{The protocol:}
  \begin{enumerate}[itemsep=0cm]
    \item The setup algorithm samples
    $\pp_{\mathrm{CM}}^G\gets
    \Setup_{\mathrm{CM}}(\secparam,S)$,
    $\pp_{\mathrm{CM}}^F\gets
    \Setup_{\mathrm{CM}}(\secparam,S\cdot d(m))$,
    $\pp_{\mathrm{IP}}\gets\Setup_{\mathrm{IP}}(\secparam)$, and
    $\pp_{\kil}\gets\kilsetup(\secparam)$, and outputs
    $\pp=(\pp_{\mathrm{CM}}^G,\pp_{\mathrm{CM}}^F,\pp_{\mathrm{IP}},\pp_{\kil})$.

    \item The prover simulates the accepting computation $C^O(x)=1$
    and records the oracle queries $x_1,\ldots,x_S$ and corresponding
    answers $y_i=O(x_i)$ for every $i\in[S]$. It constructs a
    degree-$S$ polynomial $G$ satisfying $G(i)=x_i$ for every
    $i\in[S]$, and the unique degree-$S\cdot d(m)$ polynomial $F$
    satisfying $F(z)=O(G(z))$ for every $z\in\mathbb{F}$ (here we naturally encode oracle outputs in $\mathbb{F}$); constructing
    $F$ requires at most $S\cdot d(m)$ oracle queries. The prover computes

    \[
      (\cval_G,\st_G)\gets\Commit(\pp_{\mathrm{CM}}^G,G)
      \quad\text{and}\quad
      (\cval_F,\st_F)\gets\Commit(\pp_{\mathrm{CM}}^F,F),
    \]
    and sends $(\cval_G,\cval_F)$ to the verifier.

    \item The prover and verifier engage in Kilian's protocol for the
    relation
    \[
    \begin{split}
      \Rel=\big\{&
      ((\pp_{\mathrm{CM}}^G,\pp_{\mathrm{CM}}^F,C,x,\cval_G,\cval_F),
       (G,\st_G,F,\st_F)):\,
      \Open(\pp_{\mathrm{CM}}^G,\cval_G,G,\st_G)=1,\\
      &\Open(\pp_{\mathrm{CM}}^F,\cval_F,F,\st_F)=1
      \ \wedge\ C^{G\to F}(x)=1
      \big\},
    \end{split}
    \]
    where $C^{G\to F}(x)=1$ means that, on input $x$, the $i$th
    oracle query made by $C$ is $G(i)$, its answer is taken to be
    $F(i)$, and the resulting computation accepts. The verifier
    rejects if the Kilian verifier rejects.

    \item The verifier samples $u\gets\mathbb{F}$ uniformly at random
    and sends $u$ to the prover. The prover responds with
    $v_G=G(u)$ and $v_F=F(u)$.

    \item The prover and verifier run the polynomial commitment
    opening protocol with public parameters $\pp_{\mathrm{IP}}$ to
    prove that the polynomials committed by $\cval_G$ and $\cval_F$
    evaluate to $v_G$ and $v_F$, respectively, at $u$. The verifier
    rejects if the polynomial commitment verifier rejects.

    \item The verifier accepts if the Kilian verifier and polynomial
    commitment verifier accept and $v_F=O(v_G)$, and rejects
    otherwise.
  \end{enumerate}

\end{minipage}
}
\caption{Doubly-efficient interactive argument for circuits with access to a low-degree oracle}
\label{fig:fig-low-deg}
\end{figure}

\begin{proof}

The protocol is stated in \cref{fig:fig-low-deg}. We analyze its efficiency, completeness, and soundness as
follows.

\paragraph{Efficiency.} The communication of this protocol is composed of:

\begin{itemize}
    \item Public parameters $\pp$ of size $\poly(\secp)$.
    \item Commitments $(\cval_G, \cval_F)$ of size $\poly(\secp)$.
    \item Kilian's protocol communication, of size $\poly(\secp)$.
    \item The challenge $u$ of size $\log |\mathbb{F}| = O(m + \lambda)$.
    \item The openings $(v_G, v_F)$ of size $O(m + \lambda)$.
    \item The polynomial commitment opening, of size $\poly(\secp) \cdot m$.
\end{itemize}

Therefore, the total communication is $m \cdot \poly(\secp)$.

For round complexity, since the Kilian's protocol has $O(1)$ rounds, and PC local opening argument has $O(\log (S \cdot d(m))) = O(\log n)$ rounds, the total number of rounds is $O(\log n)$.

The computation of the verifier is composed of Kilian's protocol verifier, polynomial commitment verifier and checking $v_F = O(v_G)$. Therefore the running time is $\poly(\secp) \cdot (n+m)$, and the query complexity is $1$.

\paragraph{Completeness.} Completeness follows from completeness of the polynomial commitment scheme and Kilian's protocol.

\paragraph{Soundness.} Let $\MProver$ be any malicious prover. To prove computational soundness of this protocol, we rewind $\MProver$ (just after sending $(\cval_G, \cval_F)$) to obtain a prover for the succinct argument of knowledge $\MProver_{1}$, and extract the witness from the invocation of Kilian's protocol. Specifically, we run $\kilext^{\MProver_{1}}$ to extract $(G^*, \st_G^*, F^*, \st_F^*)$.

Since the protocol also performs Kilian's protocol verification, by definition of knowledge extraction, we know that if $\MProver$ convinces the verifier, then $G^*, F^*$ extracted by $\kilext^{\MProver_{1}}$ must be valid, in the sense that they are the polynomials underlying $\cval_G, \cval_F$, and $C(x)$ accepts when making queries by $G^*$ and using oracle answers from $F^*$, which we denote by $C^{G^* \to F^*}(x) = 1$.

\begin{claim}
    For any poly-time prover $\MProver$, there exists a negligible function $\negl(\cdot)$ such that
    \[
    \probcond{
        \langle \MProver, \Verifier \rangle(\pp, C, x) = 1 \land \\
        \big( \Open(\pp_{\text{CM}}^G, \cval_G, G^*, \st_G^*) = 0 \\
        \lor \Open(\pp_{\text{CM}}^F, \cval_F, F^*, \st_F^*) = 0 \\
         \lor C^{G^* \to F^*}(x)= 0 \big)
    }{
        \pp \gets \Setup(\secparam, S) \\
        (G^*, \st_G^*, F^*, \st_F^*) \gets \kilext^{\MProver_{1}}
    } \le \negl(\secp).
    \]
\end{claim}

Then we claim that our protocol ensures $F^* = O(G^*)$ with all but negligible probability.

\begin{claim}
    For any poly-time prover $\MProver$, there exists a negligible function $\negl(\cdot)$ such that
    \[
    \probcond{
        \langle \MProver, \Verifier \rangle(\pp, C, x) = 1 \land F^* \not = O(G^*)
    }{
        \pp \gets \Setup(\secparam, S) \\
        (G^*, \st_G^*, F^*, \st_F^*) \gets \kilext^{\MProver_{1}}
    } \le \negl(\secp).
    \]
\end{claim}

\begin{proof}
    We bound the probability by considering whether $F^*(u) = O(G^*(u))$.

    First, a valid opening under $\pp_{\text{CM}}^G$ (resp.\ $\pp_{\text{CM}}^F$) enforces $\deg(G^*) \le S$ (resp.\ $\deg(F^*) \le S \cdot d(m)$), so both $F^*$ and $O \circ G^*$ are polynomials over $\mathbb{F}$ of degree at most $S \cdot d(m)$. Hence, if $F^* \neq O(G^*)$, by Schwartz-Zippel they can agree on at most $S \cdot d(m)$ points. Therefore, as $u$ is sampled uniformly from $\mathbb{F}$, we have
    \[
    \probcond{
        \langle \MProver, \Verifier \rangle(\pp, C, x) = 1 \\
        \land F^* \not = O(G^*) \land F^*(u) = O(G^*(u))
    }{
        \pp \gets \Setup(\secparam, S) \\
        (G^*, \st_G^*, F^*, \st_F^*) \gets \kilext^{\MProver_{1}}
    } \le \frac{S \cdot d(m)}{2^\secp}.
    \]

    Otherwise, if $F^*(u) \neq O(G^*(u))$, then $\MProver$ must break the polynomial commitment binding property to convince the verifier. In particular, either $v_G \neq G^*(u)$ or $v_F \neq F^*(u)$. Therefore, we can construct an adversary that breaks the binding property, by honestly computing $G^*(u), F^*(u)$ and providing their openings. This gives us that for some $\negl(\cdot)$,
    \[
    \probcond{
        \langle \MProver, \Verifier \rangle(\pp, C, x) = 1 \\
        \land F^* \not = O(G^*) \land F^*(u) \neq O(G^*(u))
    }{
        \pp \gets \Setup(\secparam, S) \\
        (G^*, \st_G^*, F^*, \st_F^*) \gets \kilext^{\MProver_{1}}
    } \le \negl(\secp).
    \]

    Combining the two cases, we conclude the proof.
\end{proof}

Now since $C^{G^* \to F^*}(x) = 1$ and $F^* = O(G^*)$, we have that $C^{O}(x) = 1$. Therefore,

\[
    \probcond{
        \langle \MProver, \Verifier \rangle(\pp, C, x) = 1 \land \\
        C^O(x) = 0
    }{
        \pp \gets \Setup(\secparam, S)
    } \le \negl(\secp).
\]
\end{proof}

\section{Conclusion}\label{sec:conclusion}

In this paper, we initiated the study of single-prover interactive proofs for oracle-aided computations as an approach to scalable AI safety. While there do not exist interactive proofs for all oracle-aided computations, we presented relativizing doubly-efficient single-prover interactive proofs and arguments for two natural settings: (1) where the computation is robust and (2) where the oracle is a low-degree polynomial. We end on some interesting further directions of study.

\begin{itemize}

\item \textbf{Experimental validation.} 
One important direction for future work is experimental validation of our protocols. For example, it would be interesting to study which realistic scalable oversight tasks satisfy our robustness assumption and whether our protocols can be implemented efficiently in practice.

\item \textbf{Other settings.} In this work, we considered the settings of robust computation and low-degree oracles. Are there other natural assumptions that one could make on the computation or on the oracle that would allow us to obtain relativizing single-prover interactive proofs? For example, it would be interesting to consider a broader class of ``learnable'' oracles. 

\item \textbf{Scope of robust computation.} As mentioned previously, many natural tasks may be able to be made robust using redundancy. Can we formalize this intuition as some theoretical model (and then apply our protocol for robust computation)?


\item \textbf{Verifying tasks that we don't know how to do.} In our model the verifier has access to the computation that he wants the prover to compute on some input; he is just too weak to perform the computation by himself. It would be interesting to study a model where not only can the verifier not perform the computation by himself, but he does not even have access to a description of the computation---this would, for example, capture the problem of scalable oversight in the setting where we are not able to perform the task that we are training the AI to do. Could we construct protocols for this setting?
    
\end{itemize}

\section{Acknowledgments}\label{sec:acknowledgments}
The authors would like to acknowledge support provided by the UK AISI Alignment Project. Zoe Xi is supported by an Akamai Presidential Fellowship.

\printbibliography

@string{springer =              "Springer"}

@string{jacm =                  "Journal of the {ACM}"}

@string{acm =                   "Association for Computing Machinery"}

@inproceedings{brown2024scalable,
  title={Scalable AI safety via doubly-efficient debate},
  author={Brown-Cohen, Jonah and Irving, Geoffrey and Piliouras, Georgios},
  booktitle={Proceedings of the 41st International Conference on Machine Learning},
  pages={4585--4602},
  year={2024}
}

@article{irving2018ai,
  title={AI safety via debate},
  author={Irving, Geoffrey and Christiano, Paul and Amodei, Dario},
  journal={arXiv preprint arXiv:1805.00899},
  year={2018}
}

@article{brown2025avoiding,
  title={Avoiding Obfuscation with Prover-Estimator Debate},
  author={Brown-Cohen, Jonah and Irving, Geoffrey and Piliouras, Georgios},
  journal={arXiv preprint arXiv:2506.13609},
  year={2025}
}

@article{buhl2025alignment,
  title={An alignment safety case sketch based on debate},
  author={Buhl, Marie Davidsen and Pfau, Jacob and Hilton, Benjamin and Irving, Geoffrey},
  journal={arXiv preprint arXiv:2505.03989},
  year={2025}
}

@article{christiano2018supervising,
  title={Supervising strong learners by amplifying weak experts},
  author={Christiano, Paul and Shlegeris, Buck and Amodei, Dario},
  journal={arXiv preprint arXiv:1810.08575},
  year={2018}
}

@article{leike2018scalable,
  title={Scalable agent alignment via reward modeling: a research direction},
  author={Leike, Jan and Krueger, David and Everitt, Tom and Martic, Miljan and Maini, Vishal and Legg, Shane},
  journal={arXiv preprint arXiv:1811.07871},
  year={2018}
}

@article{amit2024models,
  title={Models that prove their own correctness},
  author={Amit, Noga and Goldwasser, Shafi and Paradise, Orr and Rothblum, Guy},
  journal={arXiv preprint arXiv:2405.15722},
  year={2024}
}

@article{hammond2024neural,
  title={Neural interactive proofs},
  author={Hammond, Lewis and Adam-Day, Sam},
  journal={arXiv preprint arXiv:2412.08897},
  year={2024}
}

@article{anil2021learning,
  title={Learning to give checkable answers with prover-verifier games},
  author={Anil, Cem and Zhang, Guodong and Wu, Yuhuai and Grosse, Roger},
  journal={arXiv preprint arXiv:2108.12099},
  year={2021}
}

@misc{wäldchen2024interpretabilityguaranteesmerlinarthurclassifiers,
      title={Interpretability Guarantees with Merlin-Arthur Classifiers}, 
      author={Stephan Wäldchen and Kartikey Sharma and Berkant Turan and Max Zimmer and Sebastian Pokutta},
      year={2024},
      eprint={2206.00759},
      archivePrefix={arXiv},
      primaryClass={cs.LG},
      url={https://arxiv.org/abs/2206.00759}, 
}

@article{kirchner2024prover,
  title={Prover-verifier games improve legibility of llm outputs},
  author={Kirchner, Jan Hendrik and Chen, Yining and Edwards, Harri and Leike, Jan and McAleese, Nat and Burda, Yuri},
  journal={arXiv preprint arXiv:2407.13692},
  year={2024}
}

@article{goldwasser2015delegating,
  title={Delegating computation: interactive proofs for muggles},
  author={Goldwasser, Shafi and Kalai, Yael Tauman and Rothblum, Guy N},
  journal={Journal of the ACM (JACM)},
  volume={62},
  number={4},
  pages={1--64},
  year={2015},
  publisher={ACM New York, NY, USA}
}

@misc{berger2025efficientlybatchingunambiguousinteractive,
      title={Efficiently Batching Unambiguous Interactive Proofs}, 
      author={Bonnie Berger and Rohan Goyal and Matthew M. Hong and Yael Tauman Kalai},
      year={2025},
      eprint={2510.19075},
      archivePrefix={arXiv},
      primaryClass={cs.CC},
      url={https://arxiv.org/abs/2510.19075}, 
}

@inproceedings{rothblum2013interactive,
  title={Interactive proofs of proximity: delegating computation in sublinear time},
  author={Rothblum, Guy N and Vadhan, Salil and Wigderson, Avi},
  booktitle={Proceedings of the forty-fifth annual ACM symposium on Theory of computing},
  pages={793--802},
  year={2013}
}

@inproceedings{rothblum2020batch,
  title={Batch verification and proofs of proximity with polylog overhead},
  author={Rothblum, Guy N and Rothblum, Ron D},
  booktitle={Theory of Cryptography Conference},
  pages={108--138},
  year={2020},
  organization={Springer}
}

@online{irving2025relativise,
  author       = {Geoffrey Irving and Simon Marshall},
  title        = {The need to relativise in debate},
  year         = {2025},
  month        = {Jun},
  day          = {26},
  url          = {https://www.alignmentforum.org/posts/XycoFucvxAPhcgJQa/the-need-to-relativise-in-debate-1},
  note         = {AI Alignment Forum post},
  urldate      = {2025-06-26},
}

@online{barnes2020obfuscated,
  author       = {Barnes, Beth},
  title        = {Debate update: Obfuscated arguments problem},
  year         = {2020},
  month        = {Dec},
  day          = {22},
  url          = {https://www.alignmentforum.org/posts/PJLABqQ962hZEqhdB/debate-update-obfuscated-arguments-problem},
  note         = {AI Alignment Forum post},
  urldate      = {2025-06-26},
}

@article{shamirIP=PSPACE,
  author = {Shamir, Adi},
  title = {IP = PSPACE},
  year = {1992},
  issue_date = {Oct. 1992},
  publisher = {Association for Computing Machinery},
  address = {New York, NY, USA},
  volume = {39},
  number = {4},
  issn = {0004-5411},
  url = {https://doi.org/10.1145/146585.146609},
  doi = {10.1145/146585.146609},
  journal = {J. ACM},
  month = oct,
  pages = {869-877},
  numpages = {9}
}

@article{arora1998probabilistic,
  title={Probabilistic checking of proofs: A new characterization of NP},
  author={Arora, Sanjeev and Safra, Shmuel},
  journal={Journal of the ACM (JACM)},
  volume={45},
  number={1},
  pages={70--122},
  year={1998},
  publisher={ACM New York, NY, USA}
}

@inproceedings{DBLP:conf/tcc/BarbaraCG25,
  author       = {Annalisa Barbara and
                  Alessandro Chiesa and
                  Ziyi Guan},
  title        = {Relativized Succinct Arguments in the {ROM} do not Exist},
  booktitle    = {{TCC} {(1)}},
  series       = {Lecture Notes in Computer Science},
  volume       = {16268},
  pages        = {400--416},
  publisher    = {Springer},
  year         = {2025}
}

@article{micali2000computationally,
  title={Computationally sound proofs},
  author={Micali, Silvio},
  journal={SIAM Journal on Computing},
  volume={30},
  number={4},
  pages={1253--1298},
  year={2000},
  publisher={SIAM}
}

@inproceedings{kilian1992note,
  title={A note on efficient zero-knowledge proofs and arguments},
  author={Kilian, Joe},
  booktitle={Proceedings of the twenty-fourth annual ACM symposium on Theory of computing},
  pages={723--732},
  year={1992}
}

@inproceedings{GoldwasserMR85,
  author       = {Shafi Goldwasser and
                  Silvio Micali and
                  Charles Rackoff},
  editor       = {Robert Sedgewick},
  title        = {The Knowledge Complexity of Interactive Proof-Systems (Extended Abstract)},
  booktitle    = {Proceedings of the 17th Annual {ACM} Symposium on Theory of Computing,
                  May 6-8, 1985, Providence, Rhode Island, {USA}},
  pages        = {291--304},
  publisher    = {{ACM}},
  year         = {1985},
  url          = {https://doi.org/10.1145/22145.22178},
  doi          = {10.1145/22145.22178},
  bibsource    = {dblp computer science bibliography, https://dblp.org}
}

@inproceedings{Babai85,
  author       = {L{\'{a}}szl{\'{o}} Babai},
  editor       = {Robert Sedgewick},
  title        = {Trading Group Theory for Randomness},
  booktitle    = {Proceedings of the 17th Annual {ACM} Symposium on Theory of Computing,
                  May 6-8, 1985, Providence, Rhode Island, {USA}},
  pages        = {421--429},
  publisher    = {{ACM}},
  year         = {1985},
  url          = {https://doi.org/10.1145/22145.22192},
  doi          = {10.1145/22145.22192},
  bibsource    = {dblp computer science bibliography, https://dblp.org}
}

@article{DBLP:journals/ftsec/Thaler22,
  author       = {Justin Thaler},
  title        = {Proofs, Arguments, and Zero-Knowledge},
  journal      = {Found. Trends Priv. Secur.},
  volume       = {4},
  number       = {2-4},
  pages        = {117--660},
  year         = {2022},
  url          = {https://doi.org/10.1561/3300000030},
  doi          = {10.1561/3300000030},
  bibsource    = {dblp computer science bibliography, https://dblp.org}
}
    
\end{document}